%% file: main.tex
\documentclass{article} %
\usepackage{iclr2026_conference,times}

\input{0_premable.tex}

\usepackage{hyperref}
\usepackage{url}

\usepackage{xspace}
\newcommand{\ours}[0]{\texttt{CausalEvolve}\xspace}
\newcommand{\planner}[0]{\texttt{CausalPlanner}\xspace}
\newcommand{\coat}[0]{\texttt{COAT}\xspace}
\newcommand{\shinka}[0]{\texttt{ShinkaEvolve}\xspace}
\newcommand{\alevolve}[0]{\texttt{AlphaEvolve}\xspace}

\usepackage{geometry}
\usepackage{tikz}
\usepackage{caption}
\usepackage{subcaption}
\usepackage{amsmath} %
\usetikzlibrary{shapes, positioning, arrows.meta, fit, calc, backgrounds, shadows}

\definecolor{boxMem}{RGB}{225, 245, 254}   %
\definecolor{boxAct}{RGB}{255, 249, 196}   %
\definecolor{boxEnv}{RGB}{232, 245, 233}   %
\definecolor{boxWorld}{RGB}{245, 245, 245} %
\definecolor{lineColor}{RGB}{80, 80, 80}

\definecolor{memBlue}{RGB}{220, 235, 255}
\definecolor{progGreen}{RGB}{220, 255, 220}
\definecolor{obsGray}{RGB}{240, 240, 240}
\definecolor{borderGray}{RGB}{100, 100, 100}

\usepackage{tikz}
\usetikzlibrary{bayesnet, positioning, arrows.meta, shapes, calc, decorations.pathreplacing}

\definecolor{memBlue}{RGB}{220, 235, 255}
\definecolor{progGreen}{RGB}{220, 255, 220}
\definecolor{obsGray}{RGB}{240, 240, 240}
\definecolor{borderGray}{RGB}{100, 100, 100}

\title{CausalEvolve: Towards Open-Ended Discovery with Causal Scratchpad}

\author{%
   Yongqiang Chen$^{*1,2}$
   \
   Chenxi Liu
   \thanks{These authors contributed equally.}$^{~~3}$
   \ 
   Zhenhao Chen$^{1}$ 
   \ 
   Tongliang Liu$^{4,1}$ 
   \ 
   Bo Han$^{3}$
   \ 
   Kun Zhang$^{1,2}$ 
   \ 
   ~\\
    $^{1}${MBZUAI}\ $^{2}${Carnegie Mellon University}\ $^{3}${TMLR Group, Hong Kong Baptist University} \\ 
   $^{4}${SAIC Centre, The University of Sydney} \\
   \texttt{yqchen24@gmail.com}\quad \texttt{cscxliu@comp.hkbu.edu.hk}
 }

\iclrfinalcopy %
\begin{document}

\maketitle

\begin{abstract}
Evolve-based agent such as AlphaEvolve is one of the notable successes in using Large Language Models (LLMs) to build AI Scientists. These agents tackle open-ended scientific problems by iteratively improving and evolving programs, leveraging the prior knowledge and reasoning capabilities of LLMs.
  Despite the success, existing evolve-based agents lack targeted guidance for evolution and effective mechanisms for organizing and utilizing knowledge acquired from past evolutionary experience.
  Consequently, they suffer from decreasing evolution efficiency and exhibit oscillatory behavior when approaching known performance boundaries.
  To mitigate the gap, we develop \ours, equipped with a causal scratchpad that leverages LLMs to identify and reason about guiding factors for evolution.
  At the beginning, \ours first identifies outcome-level factors that offers complementary inspirations in improving the target objective.
  During the evolution, \ours also inspects surprise patterns during the evolution and abductive reasoning to hypothesize new factors, which in turn offer novel directions.
  Through comprehensive experiments, we show that \ours effectively improve the evolutionary efficiency and discovers better solutions in $4$ challenging open-ended scientific tasks.
\end{abstract}

\section{Introduction}
As large language model (LLMs) demonstrate increasing capabilities in complex and challenging reasoning tasks~\citep{guo2025deepseek,li2025system}, the community seeks to build LLM-based agents to facilitate a number of downstream applications~\citep{plaat2025agentic}. One of the most notable and promising applications is the AI Scientist agents~\citep{ZHENG2025FromAT}, where the LLM-based agent is expected to automate the scientific discovery process ranging from conducting literature surveys~\citep{wan2025deepresearch}, hypothesis generation~\citep{khemakhem2020ice}, data-driven analysis~\citep{chan2024mle-bench} to experiment design~\citep{Li2025CanLL}, etc. 
In fact, when incorporated into the agentic framework, LLMs have demonstrated great promise. \citet{lu2024ai,gottweis2025towards,Mitchener2025KosmosAA} show that LLMs can come up with new research hypotheses and proposals based on the existing literature and automate the full scientific discovery pipeline~\citep{yamada2025ai}. Recent advances in using LLMs to assist with scientific discovery shows LLMs can accelerate the idea iteration and deep literature search~\citep{Bubeck2025EarlySA,Woodruff2026AcceleratingSR}.

One of the most representative AI Scientist agents is the evolutionary coding agent, like AlphaEvolve~\citep{novikov2025alphaevolve,Lange2025ShinkaEvolveTO}. In the iterative evolutionary framework, LLMs demonstrate great capabilities in proposing, evaluating, and refining iteratively better solutions to a number of scientific problems~\citep{openevolve,Georgiev2025MathematicalEA,Cheng2025BarbariansAT}.
Despite the success, the evolution process in the existing frameworks is mainly driven by the evolution algorithm or derived from correlational studies. In contrast, human scientists can design \textit{purposeful experiments} and \textit{summarize scientific insights} from observational data~\citep{sci_revolution,kaelbling1998planning,ClarkOutline}.
The gap that emerges between the uncontrolled evolutionary process of evolve-based agents and the guided discovery process of humans raises a challenging research question:
\begin{myquotation}\centering
    \textit{How can we develop evolution-based agents to perform guided scientific discovery like humans?}
\end{myquotation}
To tackle the question, we resort to \textit{causality}, which summarizes the practice of scientific discovery of humans~\citep{spirtes2000cps,pearl2009causality}. 
Essentially, scientific discovery is about revealing the underlying causal mechanism of the interested problem~\citep{wallace1981causality,ClarkOutline}.
Hence, we can formulate the evolution-based scientific discovery process as a Partially Observable Markov Decision Process
(POMDP)~\citep{kaelbling1998planning}, where the agent needs to uncover the underlying causal mechanism through purposeful actions and interventions (Sec.~\ref{sec:formulation}).
With the POMDP formulation, we demonstrate that accumulating and guiding the evolution with \textit{causal knowledge} is crucial to both the efficiency and effectiveness of the discovery process. Without the incorporation of causality, the evolution can easily oscillate or get stuck at local optimal solutions.

To this end, we develop a new evolutionary AI Scientist framework, termed \ours, 
where we introduce a causal scratchpad to the evolution-based agent.
The guidance provided by \ours is built upon the interventional factors identified before and during the evolution process.
As the evolution-based agent primarily focuses on optimizing a target objective, such as the objective value of a combinatorial optimization problem or the accuracy of a machine learning problem~\citep{Lange2025ShinkaEvolveTO}, \ours first identifies a set of \textit{outcome-level factors} to provide complementary views of the target objective. During the evolution, \ours leverages a multi-arm bandit (MAB) to adaptively determine the desired intervention with respect to a selected outcome-level factor.

In addition, \ours also identifies \textit{procedure-level factors} from the accumulated trials with LLMs~\citep{causalcoat2024}. Intuitively, the procedure-level factors are useful interventions to the solutions that explain the objective value changes. For example, the optimization technique used to solve a combinatorial optimization problem.
Nevertheless, some combinations of apparently useful factors may lead to decreased scores, which we term as ``surprise patterns''. Understanding and explaining the ``surprise patterns'' is critical to reveal new scientific insights~\citep{wallace1981causality}. Hence, \ours also performs abductive reasoning to come up with new factors and hypothesis that will be suggested to evaluate in the future experiments to better explain all the observed patterns~\citep{sep-abduction}.

Empirically, we show that \ours significantly improves the evolution efficiency and achieves better results compared to the existing state-of-the-art \shinka~\citep{Lange2025ShinkaEvolveTO} across $4$ open-ended discovery problems.
Our contributions can be summarized as follows:
\begin{itemize}[leftmargin=*]
    \item We propose a theoretical formulation of evolution-based open-ended discovery, and demonstrate the necessity of causality (Sec.~\ref{sec:formulation});
    \item We propose a new framework \ours to realize the accumulation and guidance of causal knowledge by identifying outcome-based and procedure-based factors;
    \item \ours is shown to improve both the evolution efficiency and effectiveness across $4$ open-ended discovery problems.
\end{itemize}

\section{Related Work}

\paragraph{AI Scientist Agents.} With the significant advancement in LLM capacity and the development of Agentic system, there is a rising number of works on developing agents for helping scientific discoveries~\citep{lu2024ai,yamada2025ai,gottweis2025towards}.
One research line is to automating the pipelines in scientific activities, including literature review~\citep{huang2025deep}, hypothesis generation~\citep{li2024chain,yang2024large,wang2024scimon,yangmoose}, hypothesis verification~\citep{li2024critical,huangautomated}, and assistance in scientific reports~\citep{liang2024can}.
Another research line is to integrating the knowledge and reasoning ability of LLMs to conduct computational intensive evolution or iteration on specific scientific problems~\citep{shojaee2025llmsr,romeraparedes2024funsearch,novikov2025alphaevolve,openevolve,lange2025shinkaevolve}.
There are also works on automated tabular data analysis with machine learning workflows~\citep{zha2023tablegpt,li2023sheetcopilot,zhang2023data,lilarge}, or embodied agents that can conduct real-world experiments~\citep{roch2020chemos,zhu2022all,tom2024self,mandal2025aila}.
The impact of these lines of work has been made on scientific fields includes chemistry~\citep{yang2025multi,boiko2023autonomous}, earth science~\citep{feng2025earth}, and biology~\citep{swanson2025virtual,truhn2026artificial}.

\paragraph{Causality for Scientific Discovery.} There has been a long history for the discussions on how to understand world through observations~\citep{greenland1999causal,spirtes2000causation,pearl2009causality}. 
One research line is causal discovery for structured data, where algorithms are designed to learn directed acyclic graphs among the random variables as  causal structure, including constrained-based methods~\citep{spirtes1995causal,spirtes2000causation}, methods with constrained functional~\citep{shimizu2006linear,zhang2012identifiability,hoyer2008nonlinear}, non-stationarity~\citep{malinsky2019learning,huang2019causal,huang2020causal,liu2023causal}, the incorporation with multiple domain data~\citep{huang2020causal,yang2018characterizing,brouillard2020differentiable,mooij2020joint,perry2022causal}, and handling latent variables with the pure children assumption~\citep{li2025recovery,li2025efficient}.
Recently, there are works to integrating causality with large language models. 
One direction is to empower the causal methods with the knowledge of LLMs, which includes constructing priors based on variable descriptions~\citep{long2023causal,li2024realtcd}, adjusting the causal structure searching process~\citep{ban2023causal,vashishtha2023causal,jiralerspong2024efficient}, constructing structured variables out of unstructured data~\citep{coat2025discoveringreasoningcausalityhidden,lirevealing}, and finding valid adjustment sets for treatment effect estimation~\citep{dhawan2024end,coat2025discoveringreasoningcausalityhidden,sheth2025can}. 
Another direction is to empower LLM-based agent with causal tools for tabular data analysis~\citep{abdulaal2023causal,khatibi2024alcm,shen2024exploring,wang2025causal,verma2025causal}, revealing insights from data in an autonomous pipeline.

\input{2_POMDP.tex}

\input{3_scratchpad.tex}

\input{4_exp.tex}

\section{Conclusions}
In this work, we studied the evolutionary coding agent for scientific discovery. With the POMDP formulation of the discovery process, we demonstrate the necessity of incorporating causal knowledge. Then, we propose \ours that uses a causal scratchpad to identify and exploit outcome-based and procedure-based factors and the associated causal knowledge to guide the evolution process. Empirical results with $4$ discovery tasks verified the improved efficiency and optimality of \ours.

\clearpage

\section*{Acknowledgments}
We thank the reviewers for their constructive comments and suggestions.

\bibliography{example_paper,ref_0_ai_sci,ref_1_llm,ref_2_causality,ref_3_others}
\bibliographystyle{iclr2026_conference}

\input{9_appdx_v4.tex}

\end{document}

%% file: 0_premable.tex
\usepackage[most]{tcolorbox}
\usepackage{microtype}
\usepackage{graphicx}
\usepackage{subfigure}
\usepackage{booktabs} %
\usepackage{arydshln} %
\usepackage{tikz}
\usetikzlibrary{bayesnet} %
\usetikzlibrary{arrows}
\usepackage{amsmath}
\usepackage{amssymb}
\usepackage{mathtools}
\usepackage{amsthm}
\usepackage{pgfplots}

\usepackage{enumitem,multirow,adjustbox}

\usepackage{hyperref}
\usepackage{url}
\usepackage{xcolor}		%
\definecolor{darkblue}{rgb}{0, 0, 0.5}
\definecolor{beaublue}{rgb}{0.74, 0.83, 0.9}
\definecolor{gainsboro}{rgb}{0.86, 0.86, 0.86}
\definecolor{kleinblue}{rgb}{0,0.18,0.65}
\hypersetup{colorlinks=true,citecolor=kleinblue, linkcolor=kleinblue, urlcolor=kleinblue}

\usepackage{pifont}
\newtheorem{theorem}{Theorem}[section]

\newtheorem{definition}[theorem]{Definition}

\input{math_commands.tex}

\usepackage{enumitem,algorithm,algorithmic}

\usepackage{xspace}

\makeatletter
\newcommand*\rel@kern[1]{\kern#1\dimexpr\macc@kerna}
\newcommand*\widebar[1]{%
  \begingroup
  \def\mathaccent##1##2{%
    \rel@kern{0.8}%
    \overline{\rel@kern{-0.8}\macc@nucleus\rel@kern{0.2}}%
    \rel@kern{-0.2}%
  }%
  \macc@depth\@ne
  \let\math@bgroup\@empty \let\math@egroup\macc@set@skewchar
  \mathsurround\z@ \frozen@everymath{\mathgroup\macc@group\relax}%
  \macc@set@skewchar\relax
  \let\mathaccentV\macc@nested@a
  \macc@nested@a\relax111{#1}%
  \endgroup
}
\makeatother

\newenvironment{myquotation}{\setlength{\leftmargini}{0em}\quotation}{\endquotation}
\usepackage[hyperpageref]{backref}

\renewcommand*{\backrefalt}[4]{
  \ifcase #1 \relax
  \or
    (Cited on page #2)
  \else
    (Cited on pages #2)
  \fi
}

\definecolor{Gray}{gray}{0.9}

\graphicspath{{./fig/}}

\newtcolorbox[list inside=prompt,auto counter,number within=section]{prompt}[1][]{
    colbacktitle=black!80,
    colframe=black!80,
    coltitle=white,
    fontupper=\footnotesize,
    boxsep=5pt,
    left=0pt,
    right=0pt,
    top=0pt,
    bottom=0pt,
    boxrule=1pt,
    enhanced, 
    breakable,
    skin first=enhanced,
    skin middle=enhanced,
    skin last=enhanced,
    #1,
}

\definecolor{light-purple}{RGB}{151,156,171}
\definecolor{blue-color}{RGB}{40,166,189}
\definecolor{pink-color}{RGB}{237,46,104} 
\definecolor{dark-grey-color}{RGB}{79,91,102}

%% file: math_commands.tex
\usepackage{amsmath,amsfonts,bm}

\def\eqref#1{equation~\ref{#1}}

\def\1{\bm{1}}

\DeclareMathAlphabet{\mathsfit}{\encodingdefault}{\sfdefault}{m}{sl}
\SetMathAlphabet{\mathsfit}{bold}{\encodingdefault}{\sfdefault}{bx}{n}

\DeclareMathOperator*{\argmax}{arg\,max}

%% file: 2_POMDP.tex
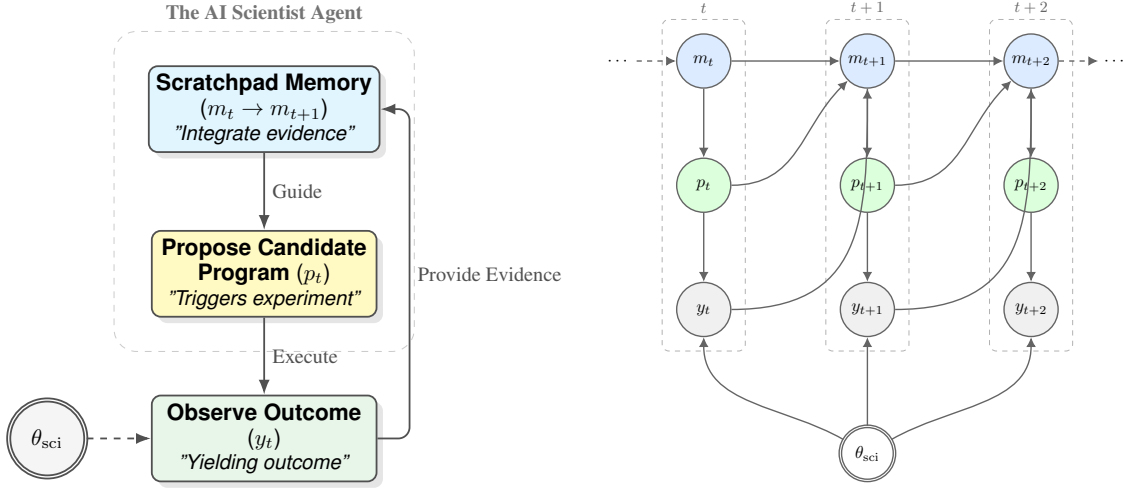
\begin{figure}[ht]
\vspace{-0.2in}
    \centering
    
    \begin{minipage}[b]{0.505\textwidth}
        \centering
        \resizebox{\linewidth}{!}{
        \begin{tikzpicture}[
            node distance=1.2cm and 1.2cm,
            every node/.style={font=\sffamily},
            block/.style={
                rectangle, 
                draw=lineColor, 
                thick, 
                rounded corners=4pt, 
                minimum width=3.5cm, 
                minimum height=1.2cm,
                align=center,
                drop shadow={opacity=0.2, shadow xshift=1mm, shadow yshift=-1mm}
            },
            world/.style={
                circle,
                draw=lineColor,
                thick,
                double,
                fill=boxWorld,
                minimum size=1.2cm,
                align=center
            },
            line/.style={->, >=LaTeX, thick, lineColor, rounded corners=5pt}
        ]

        \node[block, fill=boxMem] (mem) {
            \textbf{Scratchpad Memory}\\
            ($m_t \to m_{t+1}$)\\
            \footnotesize \textit{"Integrate evidence"}
        };

        \node[block, fill=boxAct, below=of mem] (propose) {
            \textbf{Propose Candidate}\\
            \textbf{Program} ($p_t$)\\
            \footnotesize \textit{"Triggers experiment"}
        };

        \node[block, fill=boxEnv, below=of propose] (outcome) {
            \textbf{Observe Outcome}\\
            ($y_t$)\\
            \footnotesize \textit{"Yielding outcome"}
        };

        \node[world, left=1.0cm of outcome] (theta) {$\theta_{\mathrm{sci}}$};

        \draw[line] (mem) -- node[right, font=\footnotesize] {Guide} (propose);

        \draw[line] (propose) -- node[right, font=\footnotesize] {Execute} (outcome);

        \draw[line, dashed] (theta) -- node[above, font=\footnotesize] {} (outcome);

        \draw[line] (outcome.east) -- ++(0.5,0) |- node[pos=0.25, right, font=\footnotesize] {Provide Evidence} (mem.east);

        \begin{scope}[on background layer]
            \node[
                draw=lineColor!30, dashed, 
                fit=(mem)(propose), 
                inner sep=15pt, 
                rounded corners=10pt,
                label={[text=lineColor!80]above:\textbf{The AI Scientist Agent}}
            ] {};
        \end{scope}

        \end{tikzpicture}
        }
    \end{minipage}
    \hfill
    \begin{minipage}[b]{0.475\textwidth} %
        \centering
        \resizebox{\linewidth}{!}{
        \begin{tikzpicture}[
            >={LaTeX[width=2mm,length=2mm]},
            node distance=1.4cm and 2.2cm,
            every node/.style={font=\sffamily},
            state/.style={circle, draw=borderGray, thick, minimum size=1.1cm, inner sep=0pt},
            mem/.style={state, fill=memBlue},
            prog/.style={state, fill=progGreen},
            obs/.style={state, fill=obsGray},
            global/.style={state, fill=white, double},
            plate_style/.style={draw=borderGray!50, dashed, rounded corners, inner sep=8pt, label={[text=borderGray]above:#1}},
            edge_style/.style={->, thick, borderGray}
        ]

        \node[mem] (mt1) at (0,0) {$m_{t+1}$};
        \node[prog, below=of mt1] (pt1) {$p_{t+1}$};
        \node[obs, below=of pt1] (yt1) {$y_{t+1}$};

        \node[mem, left=of mt1] (mt) {$m_t$};
        \node[prog, below=of mt] (pt) {$p_t$};
        \node[obs, below=of pt] (yt) {$y_t$};

        \node[mem, right=of mt1] (mt2) {$m_{t+2}$};
        \node[prog, below=of mt2] (pt2) {$p_{t+2}$};
        \node[obs, below=of pt2] (yt2) {$y_{t+2}$};

        \node[global, below=1.8cm of yt1] (theta) {$\theta_{\mathrm{sci}}$};

        \node[plate_style={$t$}, fit=(mt)(yt)] (plate_t) {};
        \node[plate_style={$t+1$}, fit=(mt1)(yt1)] (plate_tp1) {};
        \node[plate_style={$t+2$}, fit=(mt2)(yt2)] (plate_tp2) {};

        \foreach \t in {t, t1, t2} {
            \draw[edge_style] (m\t) -- (p\t);
            \draw[edge_style] (p\t) -- (y\t);
        }

        \draw[edge_style] (theta) edge[out=150, in=270] (yt);
        \draw[edge_style] (theta) edge[out=90,  in=270] (yt1);
        \draw[edge_style] (theta) edge[out=30,  in=270] (yt2);

        \draw[edge_style] (mt) -- (mt1);
        \draw[edge_style] (pt.east) to[out=0, in=225, looseness=1.1] (mt1.south west);
        \draw[edge_style] (yt.east) to[out=0, in=270, looseness=1.3] (mt1.south);

        \draw[edge_style] (mt1) -- (mt2);
        \draw[edge_style] (pt1.east) to[out=0, in=225, looseness=1.1] (mt2.south west);
        \draw[edge_style] (yt1.east) to[out=0, in=270, looseness=1.3] (mt2.south);

        \node[left=0.8cm of mt] (start_dots) {$\cdots$};
        \draw[edge_style, dashed] (start_dots) -- (mt);
        \node[right=0.8cm of mt2] (end_dots) {$\cdots$};
        \draw[edge_style, dashed] (mt2) -- (end_dots);

        \end{tikzpicture}
        }
    \end{minipage}
    \caption{The iterative scientific discovery loop. \textbf{Left:} Conceptual flow of the agent. The agent maintains a scratchpad memory ($m$), proposes a program ($p$), and observes the outcome ($y$) which is constrained by the unknown world state ($\theta_{\mathrm{sci}}$). The outcome feeds back into the memory for the next step. \textbf{Right:} The diagram illustrates how the AI Scientist probes the unknown world state $\theta_{\mathrm{sci}}$. By proposing a candidate program $p_t$, the agent triggers an experiment yielding outcome $y_t$. This observation provides evidence about $\theta_{\mathrm{sci}}$, which is integrated into the agent's scratchpad memory $m_{t+1}$. Over time steps $t, t+1, \dots$, this recurrent process allows the agent to navigate the performance landscape and converge towards optimal programs despite the static but unknown nature of $\theta_{\mathrm{sci}}$.}
    \label{fig:pomdp}
\end{figure}

\section{Scientific Discovery via Objective Optimization}
\label{sec:formulation}

\subsection{Formulation of Scientific Discovery}
Scientific discovery aims to uncover the underlying scientific knowledge or the causal mechanisms from interactions with the world~\citep{sci_revolution}, which can be formulated as a Partially Observed Markov Decision Process (POMDP)~\citep{kaelbling1998planning}. 

\paragraph{Scientific knowledge.} The primary objective of an AI Scientist is to uncover the \emph{underlying scientific knowledge} about the task-world, represented by a latent variable $\Theta_{\mathrm{sci}} \in \Theta$, where $\Theta$ may encode causal structure, mechanisms, inductive biases, constraints, etc. Specifically, $\Theta_{\mathrm{sci}}=\theta_{\mathrm{sci}}$ can be parameterized as a Structural Causal Model (SCM) $\theta_{\mathrm{sci}} = (\mathcal{G}, \mathcal{F}, P_U)$~\citep{spirtes2000causation}, where $\mathcal{G} = (V, E)$ is a directed graph whose nodes $V$ represent variables of interest and whose edges $E$ encode direct causal dependencies; $\mathcal{F} = \{f_v\}_{v \in V}$ is a collection of structural equations $v = f_v(\mathrm{Pa}(v), u_v)$, where $\mathrm{Pa}(v)$ denotes the parents of $v$ in $\mathcal{G}$ and $u_v$ is an exogenous noise variable; $P_U$ is a distribution over the exogenous variables $U = \{u_v\}_{v \in V}$.

\paragraph{POMDP process.} 
Given $\theta_\mathrm{sci}$, as shown in Fig.~\ref{fig:pomdp}, the AI Scientist agent, implemented via the evolutionary coding framework such as \alevolve~\citep{novikov2025alphaevolve}, will interact with the environment by proposing candidate programs $p_t\in\mathcal{P}$ (at turn $t$) to gain observations, $y_t=F(p_t,\theta_{\mathrm{sci}})$, where $F:\mathcal{P}\times\Theta\to\mathbb{R}$ is the objective that the agent aims to optimize. 
Then, the scientific discovery process can be formulated as a POMDP $\mathcal{M} = (S, A, \Omega, \mathcal{T}, O, R, \gamma)$ with a static hidden parameter as $\theta_\mathrm{sci}$ of the underlying scientific knowledge. The hidden state $s_t = \theta_{\mathrm{sci}}$ is the scientific knowledge $\theta_{\mathrm{sci}}$ that does not change over turns.
The action is $a_t = p_t$ representing the choice of which program to evaluate. The observation $o_t = y_t$ is the evaluation outcome. The transition kernel $\mathcal{T}$ can be simply considered as identity, and the observation kernel is $O(o_t \mid s_t, a_t) = P(y_t \mid \theta_{\mathrm{sci}}, p_t)$.
Given a finite experiment budget $T$, the agent chooses $p_0,\dots,p_{T-1}$ and gain observations $y_0,\dots,y_{T-1}$, so as to find $\hat p=\argmax_p F(p,\theta_{\mathrm{sci}})$ and the scientific knowledge $\theta_{\mathrm{sci}}$.

\paragraph{Evaluation as intervention on SCM.} Given the SCM parametrization of $\theta_\mathrm{sci}$, we can consider that a program $p \in \mathcal{P}$ is encoded as a particular configuration of
design variables $X = x_p$. Then, $F$ can be implemented as
\begin{equation}\label{eq:evaluation_scm}
    F(p; \theta_{\mathrm{sci}})
:= \mathbb{E}\big[\, Y \,\big|\, \mathrm{do}(X = x_p),\,
       \theta_{\mathrm{sci}} \big],
\end{equation}
i.e., the expected outcome under the intervention $\mathrm{do}(X=x_p)$ in the
true causal model $\theta_{\mathrm{sci}}$.
Typical implementations of $F$ can be the objective value of a combinatorial optimization problem, the efficiency of a kernel program, or the performance of a machine learning model~\citep{novikov2025alphaevolve}.

\paragraph{Belief as a probability distribution over $\Theta$.}
We define $b_t$ as the agent's Bayesian belief after history $h_t=\{(p_0,y_0),\dots,(p_{t-1},y_{t-1})\}$,
i.e.\ a probability distribution on $\Theta$:
\begin{equation}\label{eq:belief}
    b_t(B) = \Pr(\Theta_{\mathrm{sci}}\in B \mid h_t,e),\qquad B\subseteq \Theta.
\end{equation}
In the ideal Bayesian formalism, the belief $b_t(\theta)$ is a sufficient
statistic for decision-making~\citep{kaelbling1998planning}. In practice, the AI Scientist maintains an internal belief, which is usually implemented as memory $m_t = \Phi(h_t)$ for some (possibly learnable) summarization function $\Phi$~\citep{lange2025shinkaevolve}, to represent the \emph{approximate representation} of its knowledge about $\theta_{\mathrm{sci}}$ and the
landscape of $F(\cdot; \theta_{\mathrm{sci}})$.
Each evaluation step $(p_t, y_t)$ thus updates $m_t$, which in turn updates the
agent's effective belief about $\theta_{\mathrm{sci}}$. In this sense, each
step \emph{reveals part of the underlying scientific knowledge}, which in turn determines the next action $p_{t+1}$.

\subsection{Essentiality of Causal Knowledge for AI Scientists}
\label{sec:essentiality}
If the objective function $F$ is static universally, then with more experiment turns, the optimized solution $p_t$ and the agent's revealed scientific knowledge can also be applied universally. 
However, the observation from the evaluation is usually only given by a proxy knowledge $\theta_e$ about the scientific knowledge $\Theta_\text{sci}$ at some specific environment $e\in\mathcal{E}$.
For example, the performance of a machine learning model is usually assessed on finite samples from the test distribution, and there also exist distribution shifts from the test distribution when deploying the model in the real world~\citep{datasetshift}.
Different from $\Theta_{\mathrm{sci}}$ that characterizes the complete causal structure about the scientific problem, optimization under environment $\theta_e$ may introduce some spurious correlations that maximize the objective value $F_e$~\citep{feat}.
Therefore, without loss of generality, to retain the optimality of $\hat p$ beyond the source environment $e_\mathrm{src}$ to some target $e_\mathrm{tgt}$, it is essential to reveal the causal knowledge and answer causal questions for an AI Scientist.
\begin{definition}[Causal AI Scientist]
\label{def:causal-ai-scientist}
A \emph{Causal AI Scientist} is an agent specified by:
(i) a policy $\pi_t(\cdot\mid \theta_t,e_{\mathrm{src}})$ selecting $p_t$,
(ii) a counterfactual / explanatory operator $\mathrm{CF}$, that answer interventional queries $(e,p)$ via $\mathrm{CF}(\theta_t;e,p)$ as an ``explanation'' of predicted performance, where $\theta_t$ is the knowledge revealed at turn $t$.
\end{definition}
Without the revealing of the causal knowledge, the discovery process suffers from significant inefficiency and suboptimality issues. We discuss the two issues more concretely below.

\paragraph{Evolutionary efficiency of Causal AI Scientist.}
We begin by considering a static environment and finite $\mathcal{P}=\{p_1,\dots,p_K\}$. 
For $\theta_\mathrm{sci}$, we assume each program $p$ has a known feature vector $x_p\in\mathbb{R}^d$ with
$\|x_p\|_2\le 1$, and the unknown scientific parameter is a weight vector
$w^\star\in\mathbb{R}^d$ and $F(p)=\langle x_p, w^\star\rangle.$
Each evaluation returns a noisy observation $y_t = F(p_t)+\varepsilon_t$ where $\varepsilon_t\sim\mathcal{N}(0,\sigma^2)\ \text{i.i.d.}$. A Causal AI Scientist in this environment can be implemented via estimating the $w^\star$ and optimizing for $\hat p$ from the history.

In addition, we also consider a black-box baseline that does not consider the interactions between the historical observations. It can be characterized as the following $\theta_{\mathrm{bb}}:= \Big\{\mu:\mathcal{P}\to\mathbb{R}\Big\}$ where each program has an unrelated unknown mean $F(p;\mu)=\mu(p),$ and $y_t = \mu(p_t)+\varepsilon_t$, where $\varepsilon_t$ is the same Gaussian noise.

\begin{theorem}[Informal]
\label{thm:static-sample-efficiency-gap}
Under the given environment, there exists a policy $\pi_{\mathrm{causal}}$ such that with probability at least $1-\delta$, $F(\hat p;\theta_\mathrm{sci})$ obtains less than $2\epsilon$ error than the optimal value, with $O(d\log(K))$ turns; In contrast, the black-box baseline needs $O(K)$.
\end{theorem}
The formal description of the sample efficiency issue and the proof are given in Appendix~\ref{appdx:thm1-static}.
Theorem~\ref{thm:static-sample-efficiency-gap} shows that, when $K\gg d$, which is usually the case as the space for all programs is significantly larger than the underlying SCM, encoding (correct) causal structure yields an exponential
(or at least multiplicative) gain in sample efficiency under finite budgets.

\paragraph{Generalizability of Causal AI Scientist.} To show the necessity of capturing $\theta_\mathrm{sci}$, we have the following: 
\begin{theorem}\label{thm:nonidentifiability-barrier-rev}
    Consider the $e_{\mathrm{src}},e_{\mathrm{tgt}}\in\mathcal{E}$ and $\theta_0,\theta_1\in\Theta$ such that $F_{e_{\mathrm{src}}}(\cdot\mid p,\theta_0)=F_{e_{\mathrm{src}}}(\cdot\mid p,\theta_1)$ $\quad \forall p\in\mathcal{P}$, and $\exists\, p,p'\in\mathcal{P}\ \text{s.t.}\ 
F_{e_{\mathrm{tgt}}}(p;\theta_0)-F_{e_{\mathrm{tgt}}}(p';\theta_0)\ge\Delta
\ \text{and}\
F_{e_{\mathrm{tgt}}}(p';\theta_1)-F_{e_{\mathrm{tgt}}}(p;\theta_1)\ge\Delta$, for some $\Delta>0$, then for any policy $\pi$ that can interact only with $e_\mathrm{src}$, there exists $i\in\{0,1\}$ such that for every budget $T$, $\max_{p\in\mathcal{P}}F_{e_{\mathrm{tgt}}}(p;\theta_i)
 - F_{e_{\mathrm{tgt}}}(\hat p;\theta_i)\geq \Delta/2$.
\end{theorem}
The formal description of the generalizability issue and the proof are given in Appendix~\ref{appdx:thm_shifts}. 
Intuitively, Theorem~\ref{thm:nonidentifiability-barrier-rev} imply that if the source environment does not distinguish the corresponding $\theta_\mathrm{sci}$ among $\{\theta_0,\theta_1\}$, then the solution $\hat p$ solved given source environment is always suboptimal.
In the real world, it is usually the case that two machine learning models will have similar performances under the public test benchmarks, but exhibit significantly different behaviors when generalizing to distributions from other environments.

%% file: 3_scratchpad.tex
\section{Causal Scratchpad for Evolutionary Coding Agent}
\label{sec:method}

Given the limitations shown in Sec.~\ref{sec:essentiality}, it is essential to explicitly incorporate the causal knowledge into the evolutionary process. Hence, we present \ours, which incorporates a causal scratchpad to identify critical factors and exploit their causal relations with the objective variables to guide the evolution process.
Specifically, we consider incorporating the outcome-level factors and the procedure-level factors to tackle the efficiency and the suboptimality issues, respectively.

\subsection{Outcome-level Factor}
\label{sec:outcome_factor}
Essentially, the underlying configurations of the program can be reflected and recognized from task-dependent, real-valued descriptors extracted from the \emph{observable outcomes} of program execution. As shown in Theorem~\ref{thm:static-sample-efficiency-gap}, intervening on the underlying configuration variables provides significantly higher sample efficiency. 

\paragraph{Factor construction.}
For a given task, a set of outcome-based factors $\mathbf{m} := (m_1, m_2, \dots, m_K)$ is specified by LLMs before the evolution. An LLM would be prompted with the basic task description, which is the same as the system prompt used in evolution, and the expected output of each program, e.g., a list of coordinates, or an $n \times n$ matrix. For each of the outcome-based factors, the LLM would define the factor name and also a excitable code that maps the program output to the factor value. We list the outcome-based factors used in our tasks in Appendix~\ref{appdx:outcome}.

\paragraph{Causal Planner with outcome-level factors.} With outcome-based factors $\mathbf{m}$, we develop \planner. Specifically, we define the action space $\mathbf{A} := \cup_{m\in\mathbf{m}} \big\{(m,+1), (m,-1)\big\}$. When applying an action $(m,d)$, the existing programs would be sorted in descending order according to $m \times d$, and then the inspiration programs would be selected from the top of them. 
In $t$-th generation, after generating each child program from its parent and the inspiration programs with action $a\in\mathbf{A}$, the reward $R_a$ could be calculated. Let the $y_c$ be the child's main target that is to be maximized, and $v_t$ be the best-so-far value of the main target. We define the reward as $R_a:=(y_c -  \tau \cdot  v_t)_+$, where $\tau \in (0,1)$. We introduce this discounter $\tau$ because improving the best-so-far result could be a rare event, and therefore cannot be fairly estimated by only a few iterations. In practice, we alternate between exploration and exploitation: random actions are taken for $K$ iterations, followed by choosing the currently best action for the next $K$ iterations.

\subsection{Procedure-level Factors}
To better capture important designs of the programs and uncover their associated causal knowledge, we also introduce procedure-level factors identified from the programs.

\paragraph{Factor construction.} We construct the procedure-level factors based on the \coat framework~\citep{coat2025discoveringreasoningcausalityhidden} that leverages LLMs to identify useful procedure factors from unstructured data. As LLMs are considered incapable of understanding causality, \citet{coat2025discoveringreasoningcausalityhidden} constructs feedback to regularize the identified factors by LLMs. Similarly, we prompt LLMs to identify factors that explain the performance differences of the performances of different programs. Then, \ours estimates an approximated average treatment effect of different factors with respect to the target objective value to provide a holistic view of the usefulness of the identified procedure-level factors. Due to the limited sample size and the existence of hidden confounders, the estimated treatment effects may contain biases, while empirically, we do not need an accurate estimation, but order-preserved quantities to provide insights.

\paragraph{Abductive reasoning.} As mainly explaining the performance differences is insufficient for revealing all factors, we also incorporate a surprise detection module and leverage LLMs to perform abductive reasoning on the potentially existing factors and hypotheses that explain the surprise patterns~\citep{sep-abduction}. The detection of surprise patterns relies on the estimated treatment effects. Since the estimation can contain biases, we focus on detecting significant shifts in the estimated effects, including the signal inverses, i.e., a positively correlated factor produces negative effects, and significant quantity shifts, i.e., a minor correlated factor produces negative effects. By explaining the surprise patterns, we are able to find the underlying confounder and better reveal the underlying $\theta_\mathrm{sci}$.

%% file: 4_exp.tex
\section{Experiments}
\label{sec:exp}

\subsection{Experimental Setting}

\paragraph{Baselines.} We mainly compare \ours with the state-of-the-art evolve-based agent \shinka~\citep{lange2025shinkaevolve} that produces the best or competitive results as \alevolve~\citep{novikov2025alphaevolve} in an sample-efficient manner. 
As \shinka also incorporates a memory module to summarize the insights from $h_t$, we also consider two additional variants, \planner with meta summary module from \shinka, and \coat, to ablate the effects of two modules in \ours.
For the LLMs, we fix to using \texttt{Grok-4.1-fast-reasoning}~\citep{xai_grok4_1_fast_2025} for fair comparisons.

\textbf{Tasks.} We evaluate our framework on four scientific discovery tasks that require optimizing code for different objectives:

\textbf{Hadamard Matrix ($n=29$).}
The goal is to construct an $n\times n$ matrix $H$ with entries in $\{\pm 1\}$ that maximizes the absolute determinant $|\det(H)|$.
For $n=29$, the best-known solution achieves
$|\det(H)| = 2^{28}\cdot 7^{12}\cdot 320$,
which we use to normalize scores to $[0,1]$ for comparability with prior work~\citep{Wang2025ThetaEvolveTL}.
This discrete optimization problem requires balancing matrix properties including row orthogonality, element balance, and determinant magnitude.

\textbf{Second Autocorrelation Inequality.}
We seek a step function $f:[-1,1]\rightarrow\mathbb{R}_{\ge 0}$ (discretized into $n=256$ steps) that minimizes the ratio
\[
R(f) = \frac{\|f*f\|_2^2}{\|f*f\|_1\,\|f*f\|_\infty},
\]
where $f*f$ denotes linear autoconvolution.
The optimal value $R(f)\ge 1.1547\ldots$ remains an open conjecture.
This continuous optimization task requires carefully shaping the function's smoothness, concentration, and sparsity.

\textbf{Circle Packing ($N=26$).}
The objective is to place $N$ circles with radii $r_i$ and centers $C_i=(x_i,y_i)$ in a unit square $[0,1]^2$ such that:
(i) no circles overlap ($\|C_i-C_j\|\ge r_i+r_j$ for all $i\neq j$),
(ii) all circles remain within the square ($r_i\le C_i^x, C_i^y\le 1-r_i$),
and (iii) the sum of radii $\sum_i r_i$ is maximized.
This geometric optimization task requires spatial reasoning about density, distribution, and boundary constraints.

\textbf{AIME Mathematical Problem Solving.}
We evaluate on the 2024 American Invitational Mathematics Examination (AIME), a challenging competition consisting of 15 problems requiring integer answers in $[000,999]$.
The task is to build an LLM-based agent that solves these problems efficiently.
Performance is measured by accuracy, while auxiliary metrics track format compliance (e.g., \texttt{\textbackslash boxed\{\}} format), cost efficiency, and stability across problems.

\paragraph{Evaluation metrics.} We run every method using $3$ random seeds ($1,2,3$) to accommodate the randomness.
To compare the efficiency and the optimality, we inspect the stepwise averaged results as well as the best result from the $3$ runs, at $4$ intermediate steps. Given the difficulty of different tasks, we inspect steps $50,100,150,200$ for Second Autocorrelation Inequality and Circle Packing, steps $20, 40, 80, 100$ for Hadamard Matrix, and steps $20, 40, 60, 80$ for AIME agent. 

\subsection{Experimental Results}
The results of the experiments are given in Table~\ref{tab:main_results}. From the results, we can find that across all tasks, \ours produce significantly better averaged results than \shinka across different tasks and steps, demonstrating the effectiveness of \ours. Notably, in AIME, \ours achieves $38.89\%$ results based on the same scaffolding agent as in \shinka. While in the original paper of \shinka, even with a more sophisticated ensemble of multiple frontier reasoning models, \shinka can only achieve a performance of $34.4\%$, demonstrating the effectiveness of \ours in breaking the state-of-the-art results in the open-ended discovery.

When comparing different variants and \ours, we can find that, across $4$ tasks, \ours maintain the overall best performances, verifying that each module is essential to the success of \ours.
Interestingly, in the majority of tasks, \coat can already produce an impressive best result, demonstrating the effectiveness of procedure-level factors for optimality. When comparing results with and without \planner, we can also find that with \planner, we can achieve better results already at early steps, demonstrating the effectiveness of outcome-based factors in sample efficiency.

\input{tables/main_table}

%% file: tables/main_table.tex
\begin{table*}[t]
\vspace{-0.15in}
\centering
\small
\caption{
\textbf{Main results across four scientific discovery tasks.}
Performance is reported at training steps 1 through 4.
For each step, we report the mean performance (Mean) and the best-so-far value (Best).
All tasks are maximization objectives.
}
\vspace{-0.1in}
\label{tab:main_results}
\resizebox{\textwidth}{!}{
\setlength{\tabcolsep}{3.5pt}
\begin{tabular}{ll|cccccccc}
\toprule
\textbf{Task} & \textbf{Method}
& \multicolumn{8}{c}{\textbf{Grok-4.1-FR}} \\
\cmidrule(lr){3-10}
& 
& \multicolumn{2}{c}{Step 1}
& \multicolumn{2}{c}{Step 2}
& \multicolumn{2}{c}{Step 3}
& \multicolumn{2}{c}{Step 4} \\
\cmidrule(lr){3-10}
& 
& Mean & Best & Mean & Best & Mean & Best & Mean & Best \\
\midrule

\multirow{4}{*}{\textbf{Hadamard Matrix} ($\uparrow$)}
& \shinka
& 0.495 & 0.533 & 0.521 & 0.540 & 0.521 & 0.540 & 0.521 & 0.540 \\

& \planner (Meta)
& 0.556 & 0.573 & 0.567 & 0.573 & 0.567 & 0.573 & 0.567 & 0.573 \\

& \coat
& 0.503 & 0.519 & 0.514 & 0.543 & 0.521 & 0.552 & 0.532 & 0.561 \\

& \ours
& 0.542 & 0.574 & 0.550 & 0.574 & 0.563 & 0.576 & \textbf{0.568} & \textbf{0.576} \\

\midrule

\multirow{4}{*}{\textbf{Second Autocorr. Inequality} ($\uparrow$)}
& \shinka
& 0.723 & 0.724 & 0.729 & 0.739 & 0.735 & 0.749 & 0.737 & 0.751 \\

& \planner (Meta)
& 0.730 & 0.745 & 0.734 & 0.749 & 0.735 & 0.750 & 0.736 & 0.750 \\

& \coat
& 0.753 & 0.770 & 0.771 & 0.783 & 0.773 & 0.783 & 0.783 & 0.786 \\

& \ours
& 0.781 & 0.800 & 0.783 & 0.805 & 0.790 & 0.809 & \textbf{0.793} & \textbf{0.809} \\

\midrule

\multirow{4}{*}{\textbf{Circle Packing} ($\uparrow$)}
& \shinka
& 2.342 & 2.431 & 2.342 & 2.431 & 2.400 & 2.435 & 2.479 & 2.500 \\

& \planner (Meta)
& 2.348 & 2.541 & 2.358 & 2.541 & 2.456 & 2.541 & 2.456 & 2.541 \\

& \coat
& 2.183 & 2.261 & 2.238 & 2.292 & 2.436 & 2.560 & 2.456 & \textbf{2.568} \\

& \ours
& 2.106 & 2.295 & 2.216 & 2.370 & 2.385 & 2.516 & \textbf{2.476} & 2.564 \\

\midrule

\multirow{4}{*}{\textbf{AIME Agent} ($\uparrow$)}
& \shinka
& 33.33 & 33.33 & 34.44 & 36.67 & 34.44 & 36.67 & 34.44 & 36.67 \\

& \planner (Meta)
& 34.44 & 36.67 & 35.56 & 36.67 & 36.67 & 40.00 & 36.67 & 40.00 \\

& \coat
& 37.78 & 43.33 & 37.78 & 43.33 & 37.78 & 43.33 & 38.89 & \textbf{43.33} \\

& \ours
& 33.33 & 36.67 & 38.89 & 40.00 & 38.89 & 40.00 & 38.89 & 40.00 \\

\bottomrule
\end{tabular}}
\vspace{-0.05in}
\end{table*}

%% file: 9_appdx_v4.tex
\newpage
\appendix
\onecolumn
\section*{LLM Use Statement}
From the research side, this work studies the use of LLMs for automated scientific discovery. From the paper writing side, we use LLMs to assist with improving the writing of this work.

\section*{Ethics Statement}
We study using LLMs to automate scientific discovery that will benefit the whole humanity and society. This work does not involve human subjects or personally identifiable information beyond public benchmarks used under their licenses.

\section{Additional Technical Details}
\label{appdx:technical}

\subsection{Notation}
\label{appdx:notation}

\begin{table}[ht]
\centering
\small
\caption{Notation used in the formulation and theorems.}
\label{tab:notation-unified}
\begin{tabular}{ll}
\hline
Symbol & Meaning \\
\hline
$\mathcal{P}$ & Program / pipeline / model space (candidate designs) \\
$K$ & Number of candidate programs, $K:=|\mathcal{P}|$ (finite in Theorem~1) \\
$p\in\mathcal{P}$ & A program to evaluate (action) \\
$\mathcal{X}$ & Design-variable space (encoding of programs) \\
$x_p\in\mathcal{X}$ & Encoding of program $p$ (e.g., design variables $X=x_p$) \\
$\Theta$ & Hypothesis space of scientific knowledge (e.g., SCMs / mechanisms) \\
$\Theta_{\mathrm{sci}}$ & Latent RV taking values in $\Theta$ (Bayesian view) \\
$\theta^\star\in\Theta$ & True (fixed but unknown) scientific knowledge instance (realization) \\
$\mu_0$ & Prior over $\Theta$ (i.e., $\Theta_{\mathrm{sci}}\sim\mu_0$) \\
$\mathcal{E}$ & Environment / protocol index set (evaluation regimes, deployments) \\
$e\in\mathcal{E}$ & Environment index; $e_{\mathrm{src}}$ source, $e_{\mathrm{tgt}}$ target \\
$F_e(p;\theta)$ & True performance in env $e$ (scalar objective) \\
$P_e(\cdot\mid p,\theta)$ & Observation model (likelihood) for evaluator output in env $e$ \\
$y_t$ & Observed evaluator outcome at round $t$ \\
$h_t$ & History $\{(p_0,y_0),\dots,(p_{t-1},y_{t-1})\}$ \\
$b_t$ & Bayesian belief/posterior over $\theta$: $b_t(\cdot)=\Pr(\Theta_{\mathrm{sci}}\in\cdot\mid h_t,e_{\mathrm{src}})$ \\
$T$ & Evaluation budget / horizon (number of program evaluations) \\
\hline
\end{tabular}
\end{table}

\input{Figures/POMDP.tex}

\subsection{Random variable, space, and realization (to avoid notation confusion)}
\label{appdx:rv-space-realization}

We use the following (standard) convention.

\textbf{(i) Hypothesis space.} $\Theta$ is a set that contains all candidate scientific-knowledge hypotheses.

\textbf{(ii) True but unknown instance.} The real world is governed by a fixed but unknown
$\theta^\star\in\Theta$.

\textbf{(iii) Bayesian view (optional but convenient).} A Bayesian agent models uncertainty by
treating $\theta^\star$ as a realization of a latent random variable $\Theta_{\mathrm{sci}}$ with prior
$\mu_0$, i.e.\ $\Theta_{\mathrm{sci}}\sim\mu_0$ and $\theta^\star$ is one draw from it.
The belief $b_t$ is simply the posterior distribution of $\Theta_{\mathrm{sci}}$ after seeing history $h_t$.

\textbf{(iv) Does scientific knowledge change across environments?}
In our formulation, the \emph{underlying} scientific knowledge $\theta^\star$ is static across rounds.
Different environments $e\in\mathcal{E}$ represent different evaluation/deployment protocols
(distribution shifts, constraint changes, measurement noise, private vs public tests, etc.).
Formally, environments affect either the true performance map $F_e(\cdot;\theta)$ and/or the observation
kernel $P_e(\cdot\mid p,\theta)$, while $\theta^\star$ itself remains the same hidden instance.

\subsection{Evaluator as an observation model (covers deterministic and stochastic evaluators)}
\label{appdx:evaluator-kernel}

Fix an environment $e\in\mathcal{E}$.
When the agent evaluates program $p$, it receives an observation $y\in\mathcal{Y}$ drawn from
\[
y \sim P_e(\cdot\mid p,\theta^\star),
\]
where $P_e(\cdot\mid p,\theta)$ is a conditional distribution on $\mathcal{Y}$.

\paragraph{Deterministic evaluator.}
A deterministic evaluator is the special case where there exists a function $g_e$ such that
\[
P_e(\cdot\mid p,\theta)=\delta_{g_e(p;\theta)}(\cdot),
\quad\text{i.e.,}\quad
y = g_e(p;\theta^\star)\ \text{a.s.}
\]
In many program-evolution settings, the evaluator is designed to deterministically
check validity and compute an objective score (e.g., via a verifier and a scoring routine).

\paragraph{Stochastic/noisy evaluator.}
A common instantiation is additive noise:
\[
y = F_e(p;\theta^\star) + \varepsilon,\qquad \varepsilon\sim\mathcal{N}(0,\sigma^2),
\]
but our proofs only rely on the specific Gaussian form in Theorem~1.

\subsection{Belief and Bayes update: kernel form and undergraduate-friendly special cases}
\label{appdx:bayes-update}

Let $h_t=\{(p_0,y_0),\dots,(p_{t-1},y_{t-1})\}$ be the history.
The Bayesian belief (posterior) is
\[
b_t(B)=\Pr(\Theta_{\mathrm{sci}}\in B\mid h_t,e_{\mathrm{src}}),\qquad B\subseteq\Theta.
\]

\paragraph{General Bayes update (kernel form).}
After choosing $p_t$ and observing $y_t$ in $e_{\mathrm{src}}$, the posterior is
\begin{equation}
b_{t+1}(B)
=
\frac{\int_{B} P_{e_{\mathrm{src}}}(dy_t\mid p_t,\theta)\, b_t(d\theta)}
{\int_{\Theta} P_{e_{\mathrm{src}}}(dy_t\mid p_t,\theta)\, b_t(d\theta)}.
\label{eq:bayes-kernel-update-unified}
\end{equation}

\paragraph{Finite hypothesis space (sum form).}
If $\Theta=\{\theta_1,\dots,\theta_N\}$ is finite and the likelihood has a pmf
$P_{e_{\mathrm{src}}}(y_t\mid p_t,\theta_i)$, then
\[
b_{t+1}(\theta_i)
=
\frac{b_t(\theta_i)\, P_{e_{\mathrm{src}}}(y_t\mid p_t,\theta_i)}
{\sum_{j=1}^N b_t(\theta_j)\, P_{e_{\mathrm{src}}}(y_t\mid p_t,\theta_j)}.
\]

\paragraph{Continuous hypothesis space (density form).}
If $P_{e_{\mathrm{src}}}(dy\mid p,\theta)$ has a density $p_{e_{\mathrm{src}}}(y\mid p,\theta)$,
then
\[
b_{t+1}(\theta)
=
\frac{b_t(\theta)\, p_{e_{\mathrm{src}}}(y_t\mid p_t,\theta)}
{\int_{\Theta} b_t(\theta')\, p_{e_{\mathrm{src}}}(y_t\mid p_t,\theta')\, d\theta'}.
\]

\paragraph{Deterministic evaluator (indicator/filter form).}
If $y=g_{e_{\mathrm{src}}}(p;\theta)$ deterministically, then the update becomes
\[
b_{t+1}(d\theta)\ \propto\ \mathbf{1}\{g_{e_{\mathrm{src}}}(p_t;\theta)=y_t\}\, b_t(d\theta),
\]
i.e.\ the posterior is the prior restricted to hypotheses consistent with the observed outcome.

\section{Proof of Theorem~\ref{thm:static-sample-efficiency-gap} (Static sample-efficiency gap)}
\label{appdx:thm1-static}

Throughout this section we fix a \emph{single} static environment (drop $e$ from notation),
and assume $\mathcal{P}=\{p_1,\dots,p_K\}$ is finite.

\subsection{Protocol and performance criterion}

\paragraph{Experiment--then--commit protocol.}
A policy $\pi$ interacts for $T$ rounds. At each round $t=0,\dots,T-1$ it selects a program $p_t\in\mathcal{P}$
(possibly randomized) based on the past history $h_t$, then observes $y_t\in\mathbb{R}$.
After $T$ evaluations it outputs a final recommendation $\hat p\in\mathcal{P}$.

\paragraph{Simple regret.}
Let $f(p)$ denote the true mean performance of program $p$ in this environment.
Define the (random) simple regret
\begin{equation}
\mathrm{SR}_T
:= \max_{p\in\mathcal{P}} f(p)\;-\; f(\hat p).
\label{eq:sr-def-unified}
\end{equation}

\paragraph{$(\epsilon,\delta)$-correctness (uniform).}
Fix $\epsilon>0$ and $\delta\in(0,1)$.
We say a policy $\pi$ is \emph{$(\epsilon,\delta)$-correct uniformly on a hypothesis class $\mathcal{H}$}
if for every instance in $\mathcal{H}$,
\[
\Pr\big(\mathrm{SR}_T \le \epsilon\big)\ge 1-\delta.
\]
``Uniformly'' means the guarantee must hold for \emph{all} instances in the class, not only on average.

\subsection{Two hypothesis classes}

\textbf{(1) Structured (causal/scientific) linear class.}
Each program $p$ has a known feature vector $x_p\in\mathbb{R}^d$ with $\|x_p\|_2\le 1$.
The unknown instance is a weight vector $w^\star\in\mathbb{R}^d$ and
\begin{equation}
f(p)=\langle x_p, w^\star\rangle.
\label{eq:linear-class-unified}
\end{equation}
Observations follow a Gaussian noise model
\begin{equation}
y_t = f(p_t)+\varepsilon_t,\qquad \varepsilon_t\sim\mathcal{N}(0,\sigma^2)\ \text{i.i.d.}
\label{eq:noise-unified}
\end{equation}
Assume there exist $d$ basis programs $p^{(1)},\dots,p^{(d)}$ whose feature vectors are the standard basis:
\begin{equation}
x_{p^{(i)}}=e_i,\qquad i=1,\dots,d.
\label{eq:basis-unified}
\end{equation}

\textbf{(2) Unstructured black-box class (baseline).}
The unknown instance is an arbitrary vector of means
\[
\mu=(\mu_1,\dots,\mu_K)\in\mathbb{R}^K,\qquad f(p_i)=\mu_i,
\]
and observations are
\begin{equation}
y_t=\mu_{I_t}+\varepsilon_t,\qquad \varepsilon_t\sim\mathcal{N}(0,\sigma^2)\ \text{i.i.d.},
\label{eq:blackbox-noise-unified}
\end{equation}
where $I_t\in\{1,\dots,K\}$ is the index of the chosen program $p_t=p_{I_t}$.
Crucially, there is \emph{no assumed relation} between $\mu_i$ and $\mu_j$ for $i\neq j$.

\subsection{Formal statement and proof}

\begin{theorem}[Formal version of Theorem~\ref{thm:static-sample-efficiency-gap}]
\label{thm:static-gap-formal}
Fix $\epsilon>0$ and $\delta\in(0,1/4)$.

\begin{enumerate}
\item \textbf{(Upper bound under the structured linear class).}
Under \eqref{eq:linear-class-unified}--\eqref{eq:basis-unified} and \eqref{eq:noise-unified},
there exists a policy $\pi_{\mathrm{lin}}$ such that
\[
\Pr\big(\mathrm{SR}_T \le 2\epsilon\big)\ge 1-\delta
\quad\text{whenever}\quad
T \ \ge\ 2d\,\frac{\sigma^2}{\epsilon^2}\log\!\Big(\frac{2K}{\delta}\Big).
\]

\item \textbf{(Lower bound for the unstructured black-box class).}
For the black-box class \eqref{eq:blackbox-noise-unified}, any policy that is $(\epsilon,\delta)$-correct
\emph{uniformly for all} $\mu\in\mathbb{R}^K$ must satisfy
\[
T\ \ge\ (K-1)\,\frac{\sigma^2}{8\epsilon^2}\,\log\!\Big(\frac{1}{2\delta}\Big).
\]
\end{enumerate}
\end{theorem}

\begin{proof}
We prove the two parts separately.

\paragraph{Part (1): constructive upper bound (estimate $w^\star$ then commit).}
Evaluate each basis program $p^{(i)}$ exactly $n$ times (total $T=nd$).
Let $y^{(i)}_1,\dots,y^{(i)}_n$ be the observations for basis $i$, and define
\[
\hat w_i := \frac{1}{n}\sum_{j=1}^n y^{(i)}_j.
\]
By \eqref{eq:linear-class-unified}--\eqref{eq:basis-unified}, $f(p^{(i)})=w^\star_i$.
By \eqref{eq:noise-unified}, $\hat w_i\sim \mathcal{N}(w^\star_i,\sigma^2/n)$ and these coordinates are independent.

Define for any program $p$:
\[
\widehat f(p):=\langle x_p,\hat w\rangle,\qquad \hat w=(\hat w_1,\dots,\hat w_d).
\]
Then
\[
\widehat f(p)-f(p)=\langle x_p,\hat w-w^\star\rangle
\sim \mathcal{N}\!\Big(0,\ \frac{\sigma^2}{n}\|x_p\|_2^2\Big),
\]
so since $\|x_p\|_2\le 1$,
\[
\Pr\big(|\widehat f(p)-f(p)|\ge \epsilon\big)
\le 2\exp\!\Big(-\frac{n\epsilon^2}{2\sigma^2}\Big).
\]
Union bound over $K$ programs gives
\[
\Pr\Big(\max_{p\in\mathcal{P}}|\widehat f(p)-f(p)|\ge \epsilon\Big)
\le 2K\exp\!\Big(-\frac{n\epsilon^2}{2\sigma^2}\Big).
\]
Choose
\[
n\ \ge\ 2\frac{\sigma^2}{\epsilon^2}\log\!\Big(\frac{2K}{\delta}\Big),
\]
so that with probability at least $1-\delta$ we have
$\max_{p}|\widehat f(p)-f(p)|\le\epsilon$.

Now output $\hat p:=\arg\max_{p\in\mathcal{P}}\widehat f(p)$.
Let $p^\star:=\arg\max_p f(p)$.
On the above high-probability event,
\[
f(p^\star)-f(\hat p)
\le \big(f(p^\star)-\widehat f(p^\star)\big)+\big(\widehat f(\hat p)-f(\hat p)\big)
\le \epsilon+\epsilon = 2\epsilon.
\]
Thus $\Pr(\mathrm{SR}_T\le 2\epsilon)\ge 1-\delta$ for $T=nd$ as stated.

\paragraph{Part (2): lower bound for the black-box class.}
We construct $K$ hard instances and lower bound any uniformly $(\epsilon,\delta)$-correct policy.

Let the programs be $p_1,\dots,p_K$.
Define a base instance $\mu^{(0)}\in\mathbb{R}^K$:
\[
\mu^{(0)}_1 = 0,\qquad \mu^{(0)}_i=-2\epsilon\ \ (i=2,\dots,K).
\]
For each $i\in\{2,\dots,K\}$, define an alternative instance $\mu^{(i)}$:
\[
\mu^{(i)}_1 = 0,\qquad \mu^{(i)}_i=+2\epsilon,\qquad
\mu^{(i)}_j=-2\epsilon\ \ (j\notin\{1,i\}).
\]
Under $\mu^{(0)}$, the unique best program is $p_1$, and choosing any $p_i$ with $i\ge 2$ incurs regret $2\epsilon>\epsilon$.
Under $\mu^{(i)}$, the unique best program is $p_i$, and choosing $p_1$ incurs regret $2\epsilon>\epsilon$.

Let $P_0$ be the distribution of the full transcript
$\mathcal{T}:=(p_{0:T-1},y_{0:T-1},\hat p)$ under $\mu^{(0)}$,
and $P_i$ the analogous distribution under $\mu^{(i)}$.
Uniform $(\epsilon,\delta)$-correctness implies
\[
P_0(\hat p=p_1)\ge 1-\delta,
\qquad
P_i(\hat p=p_1)\le \delta \quad (i=2,\dots,K).
\]

\medskip
\noindent\textbf{Step 1: a KL lower bound from an event.}
For any event $A$ and distributions $P,Q$, one has
\[
\mathrm{KL}(P\|Q)
\ge
P(A)\log\frac{P(A)}{Q(A)} + (1-P(A))\log\frac{1-P(A)}{1-Q(A)}.
\]
Apply it with $A=\{\hat p=p_1\}$, $P=P_0$, $Q=P_i$.
Let $p:=P_0(A)\ge 1-\delta$ and $q:=P_i(A)\le\delta$.
For $\delta\in(0,1/4)$ this yields
\begin{equation}
\mathrm{KL}(P_0\|P_i)\ \ge\ \log\!\Big(\frac{1}{2\delta}\Big).
\label{eq:kl-lb-unified}
\end{equation}

\medskip
\noindent\textbf{Step 2: compute $\mathrm{KL}(P_0\|P_i)$ via number of pulls of arm $i$.}
Under $\mu^{(0)}$ and $\mu^{(i)}$, the policy is identical; only observations when playing $p_i$ differ:
\[
y\sim \mathcal{N}(-2\epsilon,\sigma^2)\ \text{under }\mu^{(0)},\qquad
y\sim \mathcal{N}(+2\epsilon,\sigma^2)\ \text{under }\mu^{(i)}.
\]
For Gaussians with equal variance, 
$\mathrm{KL}(\mathcal{N}(m_0,\sigma^2)\|\mathcal{N}(m_1,\sigma^2))=\frac{(m_0-m_1)^2}{2\sigma^2}$,
so each pull of $p_i$ contributes KL $\frac{(4\epsilon)^2}{2\sigma^2}=\frac{8\epsilon^2}{\sigma^2}$.

Let $N_i$ be the (random) number of times $p_i$ is evaluated in $T$ rounds.
Additivity of log-likelihood ratios over independent Gaussian samples yields
\begin{equation}
\mathrm{KL}(P_0\|P_i)
=
\frac{8\epsilon^2}{\sigma^2}\ \mathbb{E}_{P_0}[N_i].
\label{eq:kl-to-count-unified}
\end{equation}

\medskip
\noindent\textbf{Step 3: conclude the lower bound on $T$.}
Combine \eqref{eq:kl-lb-unified} and \eqref{eq:kl-to-count-unified}:
\[
\mathbb{E}_{P_0}[N_i]
\ \ge\
\frac{\sigma^2}{8\epsilon^2}\log\!\Big(\frac{1}{2\delta}\Big),
\qquad i=2,\dots,K.
\]
Summing over $i=2,\dots,K$ gives
\[
T=\sum_{i=1}^K N_i
\ \ge\
\sum_{i=2}^K \mathbb{E}_{P_0}[N_i]
\ \ge\
(K-1)\frac{\sigma^2}{8\epsilon^2}\log\!\Big(\frac{1}{2\delta}\Big).
\]
This completes the proof.
\end{proof}

\paragraph{Remark (deterministic evaluator).}
If $\sigma=0$, the structured linear class can recover $w^\star$ exactly from $d$ basis evaluations and achieve $\mathrm{SR}_T=0$,
while in the unstructured black-box class a uniform worst-case guarantee requires evaluating all $K$ programs at least once.

\paragraph{Reference for the black-box lower bound.}
The above is a standard change-of-measure/KL argument for best-arm identification in $K$-armed Gaussian bandits
(e.g., see classical treatments of best-arm identification lower bounds).

\section{Proof of Theorem~\ref{thm:nonidentifiability-barrier-rev} (Non-identifiability under environment shifts)}
\label{appdx:thm_shifts}

\subsection{Setup: source interaction, target evaluation, and target regret}

The agent can only interact with the \emph{source} environment $e_{\mathrm{src}}$:
\[
y_t\sim P_{e_{\mathrm{src}}}(\cdot\mid p_t,\theta^\star).
\]
After $T$ rounds it outputs a final program $\hat p$.
Performance is judged in a \emph{target} environment $e_{\mathrm{tgt}}$ via $F_{e_{\mathrm{tgt}}}(p;\theta^\star)$.
Define the target (simple) regret:
\[
\mathrm{GR}_T(\theta^\star)
:= \max_{p\in\mathcal{P}}F_{e_{\mathrm{tgt}}}(p;\theta^\star)
 - F_{e_{\mathrm{tgt}}}(\hat p;\theta^\star).
\]

\subsection{Formal statement and proof}

\begin{theorem}[Non-identifiability barrier under shifts]
\label{thm:nonidentifiability-barrier-formal}
Fix $e_{\mathrm{src}},e_{\mathrm{tgt}}\in\mathcal{E}$.
Assume there exist two hypotheses $\theta_0,\theta_1\in\Theta$ such that:

\begin{align}
&\textbf{(Source indistinguishability)}\quad
P_{e_{\mathrm{src}}}(\cdot\mid p,\theta_0)=P_{e_{\mathrm{src}}}(\cdot\mid p,\theta_1),
\quad \forall p\in\mathcal{P}. \label{eq:src-indist-unified}\\[2mm]
&\textbf{(Target optimal action flips with margin $\Delta$)}\quad
\exists\, p_0,p_1\in\mathcal{P}\ \text{and}\ \Delta>0\ \text{s.t.}\\
&\qquad p_0\in\arg\max_{p\in\mathcal{P}}F_{e_{\mathrm{tgt}}}(p;\theta_0),
\quad
p_1\in\arg\max_{p\in\mathcal{P}}F_{e_{\mathrm{tgt}}}(p;\theta_1), \nonumber\\
&\qquad F_{e_{\mathrm{tgt}}}(p_0;\theta_0)-F_{e_{\mathrm{tgt}}}(p;\theta_0)\ge \Delta,\ \forall p\neq p_0,
\quad
F_{e_{\mathrm{tgt}}}(p_1;\theta_1)-F_{e_{\mathrm{tgt}}}(p;\theta_1)\ge \Delta,\ \forall p\neq p_1.
\label{eq:tgt-flip-gap-unified}
\end{align}

Then for any policy $\pi$ that can interact only with $e_{\mathrm{src}}$,
there exists $i\in\{0,1\}$ such that for every budget $T$,
\[
\mathbb{E}\big[\mathrm{GR}_T(\theta_i)\big]\ \ge\ \Delta/2.
\]
This impossibility holds whether the evaluator is stochastic or deterministic, since
\eqref{eq:src-indist-unified} is stated at the level of the full observation model $P_{e_{\mathrm{src}}}$.
\end{theorem}

\begin{proof}
Let $\mathbb{P}_i$ be the distribution over the full transcript
\[
\mathcal{T}:=(p_{0:T-1},y_{0:T-1},\hat p)
\]
when the true hypothesis is $\theta_i$ and interaction is only with $e_{\mathrm{src}}$.

By \eqref{eq:src-indist-unified}, for any history and any chosen action $p_t$,
the conditional distribution of $y_t$ is identical under $\theta_0$ and $\theta_1$.
By induction on $t$, the entire transcript distribution is identical:
\[
\mathbb{P}_0=\mathbb{P}_1.
\]
In particular, the marginal distribution of the final output $\hat p$ is the same under $\theta_0$ and $\theta_1$.
Let this common distribution be denoted by $Q$ on $\mathcal{P}$.

Now consider the expected target regret under $\theta_0$:
by \eqref{eq:tgt-flip-gap-unified}, any output $\hat p\neq p_0$ incurs regret at least $\Delta$ under $\theta_0$:
\[
\mathrm{GR}_T(\theta_0)=F_{e_{\mathrm{tgt}}}(p_0;\theta_0)-F_{e_{\mathrm{tgt}}}(\hat p;\theta_0)\ \ge\ \Delta\cdot \mathbf{1}\{\hat p\neq p_0\}.
\]
Taking expectation w.r.t.\ $Q$ yields
\[
\mathbb{E}[\mathrm{GR}_T(\theta_0)]\ \ge\ \Delta\cdot (1-Q(\hat p=p_0)).
\]
Similarly,
\[
\mathbb{E}[\mathrm{GR}_T(\theta_1)]\ \ge\ \Delta\cdot (1-Q(\hat p=p_1)).
\]
Since $Q(\hat p=p_0)+Q(\hat p=p_1)\le 1$, at least one of these probabilities is at most $1/2$,
so at least one of the two expected regrets is at least $\Delta/2$:
\[
\max\{\mathbb{E}[\mathrm{GR}_T(\theta_0)],\ \mathbb{E}[\mathrm{GR}_T(\theta_1)]\}\ \ge\ \Delta/2.
\]
This proves the claim.
\end{proof}

\subsection{Concrete examples satisfying the conditions}
\label{appdx:thm2-examples}

We give two illustrative examples where source data cannot distinguish two hypotheses, yet the target-optimal decision differs.

\paragraph{Example 1: public test vs private (distribution shift / shortcut feature).}
Let $\theta\in\{\theta_0,\theta_1\}$ encode which feature is truly stable/causal.
Programs correspond to two model families:
$p_0$ uses a stable causal feature; $p_1$ uses a shortcut feature.
In the source environment (public benchmark), the shortcut feature is perfectly correlated with labels,
so both hypotheses yield the same evaluator distribution for every program, satisfying \eqref{eq:src-indist-unified}.
In the target environment (deployment/private), the shortcut correlation breaks:
under $\theta_0$, $p_0$ is uniquely optimal; under $\theta_1$, $p_1$ is uniquely optimal, with margin $\Delta$,
satisfying \eqref{eq:tgt-flip-gap-unified}.
No amount of interaction with $e_{\mathrm{src}}$ can identify which world holds.

\paragraph{Example 2: relaxed verification (slack) vs exact verification (constraint shift).}
In combinatorial optimization, it is common to evaluate candidate programs using a \emph{relaxed} verifier
(e.g., allowing numerical slack), then validate with an \emph{exact} verifier.
For instance, in circle packing, one may verify non-overlap with a numerical slack such as $10^{-6}$,
and later validate with an exact checker; converting a relaxed-feasible solution into an exact-feasible one
may require tiny but nonzero modifications, and rankings can change when switching verifiers.
This is explicitly discussed in the context of circle packing verification with slack vs exact validation.
The source environment $e_{\mathrm{src}}$ can correspond to the relaxed evaluator, while the target
environment $e_{\mathrm{tgt}}$ corresponds to the exact evaluator.
Then two hypotheses $\theta_0,\theta_1$ can be constructed so that they are indistinguishable under the relaxed evaluator
for all queried programs, yet the exact evaluator reverses which program is truly best (with gap $\Delta$),
matching Theorem~\ref{thm:nonidentifiability-barrier-formal}.

\section{More Details on Outcome-level Factors}
\label{appdx:outcome}
\input{tables/aux-metrics}

%% file: Figures/POMDP.tex
\begin{figure}[ht]
    \centering
    
    \begin{minipage}[b]{0.505\textwidth}
        \centering
        \resizebox{\linewidth}{!}{
        \begin{tikzpicture}[
            node distance=1.2cm and 1.2cm,
            every node/.style={font=\sffamily},
            block/.style={
                rectangle, 
                draw=lineColor, 
                thick, 
                rounded corners=4pt, 
                minimum width=3.5cm, 
                minimum height=1.2cm,
                align=center,
                drop shadow={opacity=0.2, shadow xshift=1mm, shadow yshift=-1mm}
            },
            world/.style={
                circle,
                draw=lineColor,
                thick,
                double,
                fill=boxWorld,
                minimum size=1.2cm,
                align=center
            },
            line/.style={->, >=LaTeX, thick, lineColor, rounded corners=5pt}
        ]

        \node[block, fill=boxMem] (mem) {
            \textbf{Scratchpad Memory}\\
            ($m_t \to m_{t+1}$)\\
            \footnotesize \textit{"Integrate evidence"}
        };

        \node[block, fill=boxAct, below=of mem] (propose) {
            \textbf{Propose Candidate}\\
            \textbf{Program} ($p_t$)\\
            \footnotesize \textit{"Triggers experiment"}
        };

        \node[block, fill=boxEnv, below=of propose] (outcome) {
            \textbf{Observe Outcome}\\
            ($y_t$)\\
            \footnotesize \textit{"Yielding outcome"}
        };

        \node[world, left=1.0cm of outcome] (theta) {$\theta_{\mathrm{sci}}$};

        \draw[line] (mem) -- node[right, font=\footnotesize] {Guide} (propose);

        \draw[line] (propose) -- node[right, font=\footnotesize] {Execute} (outcome);

        \draw[line, dashed] (theta) -- node[above, font=\footnotesize] {} (outcome);

        \draw[line] (outcome.east) -- ++(0.5,0) |- node[pos=0.25, right, font=\footnotesize] {Provide Evidence} (mem.east);

        \begin{scope}[on background layer]
            \node[
                draw=lineColor!30, dashed, 
                fit=(mem)(propose), 
                inner sep=15pt, 
                rounded corners=10pt,
                label={[text=lineColor!80]above:\textbf{The AI Scientist Agent}}
            ] {};
        \end{scope}

        \end{tikzpicture}
        }
    \end{minipage}
    \hfill
    \begin{minipage}[b]{0.475\textwidth} %
        \centering
        \resizebox{\linewidth}{!}{
        \begin{tikzpicture}[
            >={LaTeX[width=2mm,length=2mm]},
            node distance=1.4cm and 2.2cm,
            every node/.style={font=\sffamily},
            state/.style={circle, draw=borderGray, thick, minimum size=1.1cm, inner sep=0pt},
            mem/.style={state, fill=memBlue},
            prog/.style={state, fill=progGreen},
            obs/.style={state, fill=obsGray},
            global/.style={state, fill=white, double},
            plate_style/.style={draw=borderGray!50, dashed, rounded corners, inner sep=8pt, label={[text=borderGray]above:#1}},
            edge_style/.style={->, thick, borderGray}
        ]

        \node[mem] (mt1) at (0,0) {$m_{t+1}$};
        \node[prog, below=of mt1] (pt1) {$p_{t+1}$};
        \node[obs, below=of pt1] (yt1) {$y_{t+1}$};

        \node[mem, left=of mt1] (mt) {$m_t$};
        \node[prog, below=of mt] (pt) {$p_t$};
        \node[obs, below=of pt] (yt) {$y_t$};

        \node[mem, right=of mt1] (mt2) {$m_{t+2}$};
        \node[prog, below=of mt2] (pt2) {$p_{t+2}$};
        \node[obs, below=of pt2] (yt2) {$y_{t+2}$};

        \node[global, below=1.8cm of yt1] (theta) {$\theta_{\mathrm{sci}}$};

        \node[plate_style={$t$}, fit=(mt)(yt)] (plate_t) {};
        \node[plate_style={$t+1$}, fit=(mt1)(yt1)] (plate_tp1) {};
        \node[plate_style={$t+2$}, fit=(mt2)(yt2)] (plate_tp2) {};

        \foreach \t in {t, t1, t2} {
            \draw[edge_style] (m\t) -- (p\t);
            \draw[edge_style] (p\t) -- (y\t);
        }

        \draw[edge_style] (theta) edge[out=150, in=270] (yt);
        \draw[edge_style] (theta) edge[out=90,  in=270] (yt1);
        \draw[edge_style] (theta) edge[out=30,  in=270] (yt2);

        \draw[edge_style] (mt) -- (mt1);
        \draw[edge_style] (pt.east) to[out=0, in=225, looseness=1.1] (mt1.south west);
        \draw[edge_style] (yt.east) to[out=0, in=270, looseness=1.3] (mt1.south);

        \draw[edge_style] (mt1) -- (mt2);
        \draw[edge_style] (pt1.east) to[out=0, in=225, looseness=1.1] (mt2.south west);
        \draw[edge_style] (yt1.east) to[out=0, in=270, looseness=1.3] (mt2.south);

        \node[left=0.8cm of mt] (start_dots) {$\cdots$};
        \draw[edge_style, dashed] (start_dots) -- (mt);
        \node[right=0.8cm of mt2] (end_dots) {$\cdots$};
        \draw[edge_style, dashed] (mt2) -- (end_dots);

        \end{tikzpicture}
        }
    \end{minipage}
    \caption{The iterative scientific discovery loop. \textbf{Left:} Conceptual flow of the agent. The agent maintains a scratchpad memory ($m$), proposes a program ($p$), and observes the outcome ($y$) which is constrained by the unknown world state ($\theta_{\mathrm{sci}}$). The outcome feeds back into the memory for the next step. \textbf{Right:} The diagram illustrates how the AI Scientist probes the unknown world state $\theta_{\mathrm{sci}}$. By proposing a candidate program $p_t$, the agent triggers an experiment yielding outcome $y_t$. This observation provides evidence about $\theta_{\mathrm{sci}}$, which is integrated into the agent's scratchpad memory $m_{t+1}$. Over time steps $t, t+1, \dots$, this recurrent process allows the agent to navigate the performance landscape and converge towards optimal programs despite the static but unknown nature of $\theta_{\mathrm{sci}}$.}
\end{figure}
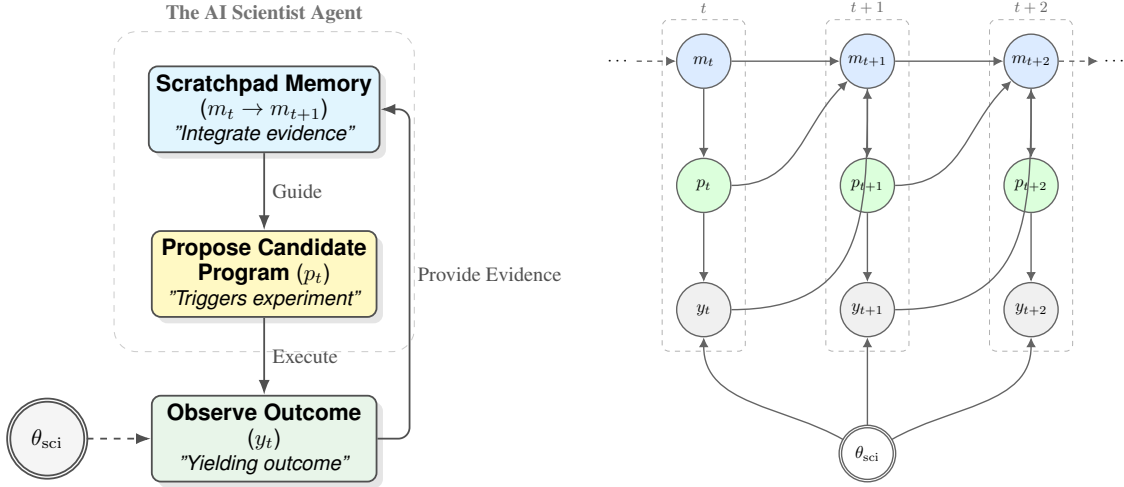

%% file: tables/aux-metrics.tex
\begin{table*}[t]
\centering
\caption{Mathematical definitions of auxiliary metrics across tasks.
All metrics are deterministic outcome-level functionals of the program outputs.
For subset-defined metrics (e.g., \texttt{large\_circle\_margin}), if the index set is empty, the metric value is defined as $0$.}
\label{tab:aux_metrics_math}
\resizebox{0.98\textwidth}{!}{
\begin{tabular}{llll}
\toprule
\textbf{Task} &
\textbf{Program Output} &
\textbf{Aux Metric} &
\textbf{Definition} \\
\midrule

\multirow{4}{*}{Hadamard Matrix}
& \multirow{4}{*}{
\begin{tabular}[c]{@{}l@{}}
$H\in\{\pm1\}^{n\times n}$ \\
binary matrix
\end{tabular}}
& row\_orthogonality\_deviation
& $\displaystyle
\frac{1}{n(n-1)}\sum_{i\neq j}\left|\sum_{k} H_{ik}H_{jk}\right|
$ \\

&& row\_sum\_variance
& $\displaystyle
\mathrm{Var}\!\left(\sum_{j} H_{ij}\right)
$ \\

&& element\_balance
& $\displaystyle
\frac{1}{n^2}\sum_{i,j}\mathbf{1}[H_{ij}=+1]
$ \\

&& log10\_abs\_det
& $\displaystyle
\log_{10}|\det(H)|
$ \\

\midrule

\multirow{6}{*}{Second Autocorr Inequality}
& \multirow{6}{*}{
\begin{tabular}[c]{@{}l@{}}
$f\in\mathbb{R}^n,\ f_i\ge 0$ \\
nonnegative discrete function
\end{tabular}}
& smoothness\_score
& $\displaystyle
\frac{1}{n-1}\sum_{i}|f_{i+1}-f_i|
$ \\

&& center\_concentration
& $\displaystyle
\sum_{|x_i|\le0.5} f_i \; / \; \sum_i f_i
$ \\

&& sparsity
& $\displaystyle
\frac{1}{n}\sum_i \mathbf{1}[f_i<\varepsilon]
$ \\

&& peak\_to\_average\_ratio
& $\displaystyle
\max_i f_i \; / \; \mathbb{E}[f]
$ \\

&& tail\_mass
& $\displaystyle
\sum_{|x_i|>0.5} f_i \; / \; \sum_i f_i
$ \\

&& entropy
& $\displaystyle
-\sum_i p_i\log p_i,
\quad p_i=f_i/\sum_j f_j
$ \\

\midrule

\multirow{7}{*}{Circle Packing}
& \multirow{7}{*}{
\begin{tabular}[c]{@{}l@{}}
$\{(C_i,r_i)\}_{i=1}^N$ \\
circle centers and radii
\end{tabular}}
& density\_score
& $\displaystyle
\sum_i \pi r_i^2 \; / \; S^2
$ \\

&& center\_spread\_index
& $\displaystyle
\frac{1}{N}\sum_i \|C_i-(S/2,S/2)\|_2
$ \\

&& radius\_std\_normalized
& $\displaystyle
\mathrm{Std}(r) \; / \; \mathbb{E}[r]
$ \\

&& neighbor\_distance\_ratio
& $\displaystyle
\frac{1}{N}\sum_i \min_{j\neq i}\|C_i-C_j\|_2 \; / \; r_i
$ \\

&& large\_circle\_margin
& $\displaystyle
\frac{1}{|I|}\sum_{i\in I}
\big(\min(C_i^x,S-C_i^x,C_i^y,S-C_i^y)-r_i\big),
\; I=\{i:r_i>\mathbb{E}[r]\}
$ \\

&& pairwise\_radii\_product\_sum
& $\displaystyle
\sum_{i<j} r_i r_j
$ \\

&& centroid\_distance\_variance
& $\displaystyle
\mathrm{Var}\!\left(\|C_i-\mathbb{E}[C]\|_2\right)
$ \\

\midrule

\multirow{5}{*}{ADAS-AIME}
& \multirow{5}{*}{
\begin{tabular}[c]{@{}l@{}}
$(\text{accuracy},\text{cost},df)$ \\
evaluation records
\end{tabular}}
& boxed\_format\_rate
& $\displaystyle
100\cdot\frac{1}{N}\sum_i
\mathbf{1}[\text{response}_i\text{ contains }\boxed{\cdot}]
$ \\

&& three\_digit\_answer\_rate
& $\displaystyle
100\cdot\frac{1}{N}\sum_i \mathbf{1}[\text{len(answer}_i)=3]
$ \\

&& cost\_efficiency
& $\displaystyle
\text{accuracy} \; / \; \sum_i \text{cost}_i
$ \\

&& accuracy\_variance
& $\displaystyle
\mathrm{Var}(\text{accuracy over sliding windows})
$ \\

&& max\_consecutive\_errors
& $\displaystyle
\max_k \sum_{t=t_k}^{t_k+\ell}\mathbf{1}[\text{incorrect}_t]
$ \\

\bottomrule
\end{tabular}
}
\end{table*}

%% file: ref_0_ai_sci.bib
@inproceedings{lilarge,
  title={Are Large Language Models Ready for Multi-Turn Tabular Data Analysis?},
  author={Li, Jinyang and Huo, Nan and Gao, Yan and Shi, Jiayi and Zhao, Yingxiu and Qu, Ge and Qin, Bowen and Wu, Yurong and Li, Xiaodong and Ma, Chenhao and others},
  booktitle={Forty-second International Conference on Machine Learning},
  year={2025}
}

@article{zhang2023data,
  title={Data-copilot: Bridging billions of data and humans with autonomous workflow},
  author={Zhang, Wenqi and Shen, Yongliang and Lu, Weiming and Zhuang, Yueting},
  journal={arXiv preprint arXiv:2306.07209},
  year={2023}
}

@article{li2023sheetcopilot,
  title={Sheetcopilot: Bringing software productivity to the next level through large language models},
  author={Li, Hongxin and Su, Jingran and Chen, Yuntao and Li, Qing and Zhang, Zhao-Xiang},
  journal={Advances in Neural Information Processing Systems},
  volume={36},
  pages={4952--4984},
  year={2023}
}

@article{zha2023tablegpt,
  title={Tablegpt: Towards unifying tables, nature language and commands into one gpt},
  author={Zha, Liangyu and Zhou, Junlin and Li, Liyao and Wang, Rui and Huang, Qingyi and Yang, Saisai and Yuan, Jing and Su, Changbao and Li, Xiang and Su, Aofeng and others},
  journal={arXiv preprint arXiv:2307.08674},
  year={2023}
}

@article{roch2020chemos,
  title={ChemOS: An orchestration software to democratize autonomous discovery},
  author={Roch, Lo{\"i}c M. and others},
  journal={PLOS ONE},
  volume={15},
  number={4},
  pages={e0229862},
  year={2020}
}

@article{mandal2025aila,
  title={Artificially Intelligent Lab Assistant for Automated Experimentation},
  author={Mandal, Shubham and others},
  journal={Nature Communications},
  volume={16},
  pages={1234},
  year={2025}
}

@article{shojaee2025llmsr,
  title={Scientific Equation Discovery via Programming with Large Language Models},
  author={Shojaee, Parshin and others},
  journal={arXiv preprint arXiv:2404.18400},
  year={2025}
}

@article{romeraparedes2024funsearch,
  title={Mathematical discoveries from program search with large language models},
  author={Romera-Paredes, Bernardino and others},
  journal={Nature},
  volume={625},
  pages={468--475},
  year={2024}
}

@article{lange2025shinkaevolve,
  title={Towards Open-Ended and Sample-Efficient Program Evolution},
  author={Lange, Robert and others},
  journal={arXiv preprint arXiv:2509.19349},
  year={2025}
}

@article{liang2024can,
  title={Can large language models provide useful feedback on research papers? A large-scale empirical analysis},
  author={Liang, Weixin and Zhang, Yuhui and Cao, Hancheng and Wang, Binglu and Ding, Daisy Yi and Yang, Xinyu and Vodrahalli, Kailas and He, Siyu and Smith, Daniel Scott and Yin, Yian and others},
  journal={NEJM AI},
  volume={1},
  number={8},
  pages={AIoa2400196},
  year={2024},
  publisher={Massachusetts Medical Society}
}

@inproceedings{wang2024scimon,
  title={Scimon: Scientific inspiration machines optimized for novelty},
  author={Wang, Qingyun and Downey, Doug and Ji, Heng and Hope, Tom},
  booktitle={Proceedings of the 62nd Annual Meeting of the Association for Computational Linguistics (Volume 1: Long Papers)},
  pages={279--299},
  year={2024}
}

@inproceedings{yang2024large,
  title={Large language models for automated open-domain scientific hypotheses discovery},
  author={Yang, Zonglin and Du, Xinya and Li, Junxian and Zheng, Jie and Poria, Soujanya and Cambria, Erik},
  booktitle={Findings of the Association for Computational Linguistics: ACL 2024},
  pages={13545--13565},
  year={2024}
}

@article{li2024chain,
  title={Chain of ideas: Revolutionizing research via novel idea development with llm agents},
  author={Li, Long and Xu, Weiwen and Guo, Jiayan and Zhao, Ruochen and Li, Xingxuan and Yuan, Yuqian and Zhang, Boqiang and Jiang, Yuming and Xin, Yifei and Dang, Ronghao and others},
  journal={arXiv preprint arXiv:2410.13185},
  year={2024}
}

@article{li2024critical,
  title={CriticAL: Critic Automation with Language Models},
  author={Li, Michael Y and Vajipey, Vivek and Goodman, Noah D and Fox, Emily B},
  journal={arXiv preprint arXiv:2411.06590},
  year={2024}
}

@article{yamada2025ai,
  title={The ai scientist-v2: Workshop-level automated scientific discovery via agentic tree search},
  author={Yamada, Yutaro and Lange, Robert Tjarko and Lu, Cong and Hu, Shengran and Lu, Chris and Foerster, Jakob and Clune, Jeff and Ha, David},
  journal={arXiv preprint arXiv:2504.08066},
  year={2025}
}

@article{lu2024ai,
  title={The ai scientist: Towards fully automated open-ended scientific discovery},
  author={Lu, Chris and Lu, Cong and Lange, Robert Tjarko and Foerster, Jakob and Clune, Jeff and Ha, David},
  journal={arXiv preprint arXiv:2408.06292},
  year={2024}
}

@article{yang2025multi,
  title={From what to why: A multi-agent system for evidence-based chemical reaction condition reasoning},
  author={Yang, Cheng and Lu, Jiaxuan and Wan, Haiyuan and Yu, Junchi and Qin, Feiwei},
  journal={ICLR},
  year={2026}
}

@article{Georgiev2025MathematicalEA,
  title={Mathematical exploration and discovery at scale},
  author={Bogdan Georgiev and Javier G'omez-Serrano and Terence Tao and Adam Zsolt Wagner},
  journal={ArXiv},
  year={2025},
  volume={abs/2511.02864},
}

@article{Bubeck2025EarlySA,
  title={Early science acceleration experiments with GPT-5},
  author={S{\'e}bastien Bubeck and Christian Coester and Ronen Eldan and Timothy Gowers and Yin Tat Lee and Alexandru Lupsasca and Mehtaab Sawhney and Robert Scherrer and Mark Sellke and Brian K. Spears and Derya Unutmaz and Kevin Weil and Steven Yin and Nikita Zhivotovskiy},
  journal={ArXiv},
  year={2025},
  volume={abs/2511.16072},
}

@article{Cheng2025BarbariansAT,
  title={Barbarians at the Gate: How AI is Upending Systems Research},
  author={Audrey Cheng and Shu Liu and Melissa Z. Pan and Zhifei Li and Bowen Wang and Alexander Krentsel and Tian Xia and Mert Cemri and Jongseok Park and Shuo Yang and Jeff Chen and Lakshya A Agrawal and Aditya Desai and Jiarong Xing and Koushik Sen and Matei Zaharia and Ion Stoica},
  journal={ArXiv},
  year={2025},
  volume={abs/2510.06189},
}

@misc{openevolve,
  title = {OpenEvolve: an open-source evolutionary coding agent},
  author = {Asankhaya Sharma},
  year = {2025},
  publisher = {GitHub},
  url = {https://github.com/algorithmicsuperintelligence/openevolve}
}

@article{Woodruff2026AcceleratingSR,
  title={Accelerating Scientific Research with Gemini: Case Studies and Common Techniques},
  author={David P. Woodruff and Vincent Cohen-Addad and Lalit Jain and Jieming Mao and Song Zuo and Mohammad Rez Bateni and Simina Br{\^a}nzei and Michael P. Brenner and Lin Chen and Ying Feng and Lance Fortnow and Gang Fu and Ziyi Guan and Zahra Hadizadeh and Mohammad Taghi Hajiaghayi and Mahdi JafariRaviz and Adel Javanmard and S. KarthikC. and Ken-ichi Kawarabayashi and Ravi Kumar and Silvio Lattanzi and Euiwoong Lee and Yi Li and Ioannis Panageas and Dimitris Paparas and Benjamin Przybocki and Bernardo Subercaseaux and Ola Svensson and Shayan Taherijam and Xuan Wu and Eylon Yogev and Morteza Zadimoghaddam and Samson Zhou and Vahab S. Mirrokni},
  year={2026},
  journal={ArXiv},
  volume={abs/2602.03837},
}

@article{wan2025deepresearch,
  title={DeepResearch Arena: The First Exam of LLMs' Research Abilities via Seminar-Grounded Tasks},
  author={Wan, Haiyuan and Yang, Chen and Yu, Junchi and Tu, Meiqi and Lu, Jiaxuan and Yu, Di and Cao, Jianbao and Gao, Ben and Xie, Jiaqing and Wang, Aoran and others},
  journal={AAAI},
  year={2026}
}

@article{swanson2025virtual,
  title={The Virtual Lab of AI agents designs new SARS-CoV-2 nanobodies},
  author={Swanson, Kyle and Wu, Wesley and Bulaong, Nash L and Pak, John E and Zou, James},
  journal={Nature},
  volume={646},
  number={8085},
  pages={716--723},
  year={2025},
  publisher={Nature Publishing Group UK London}
}

@article{truhn2026artificial,
  title={Artificial intelligence agents in cancer research and oncology},
  author={Truhn, Daniel and Azizi, Shekoofeh and Zou, James and Cerda-Alberich, Leonor and Mahmood, Faisal and Kather, Jakob Nikolas},
  journal={Nature Reviews Cancer},
  pages={1--14},
  year={2026},
  publisher={Nature Publishing Group UK London}
}

@article{feng2025earth,
  title={Earth-agent: Unlocking the full landscape of earth observation with agents},
  author={Feng, Peilin and Lv, Zhutao and Ye, Junyan and Wang, Xiaolei and Huo, Xinjie and Yu, Jinhua and Xu, Wanghan and Zhang, Wenlong and Bai, Lei and He, Conghui and others},
  journal={arXiv preprint arXiv:2509.23141},
  year={2025}
}

@article{boiko2023autonomous,
  title={Autonomous chemical research with large language models},
  author={Boiko, Daniil A and MacKnight, Robert and Kline, Ben and Gomes, Gabe},
  journal={Nature},
  volume={624},
  number={7992},
  pages={570--578},
  year={2023},
  publisher={Nature Publishing Group UK London}
}

@article{tom2024self,
  title={Self-driving laboratories for chemistry and materials science},
  author={Tom, Gary and Schmid, Stefan P and Baird, Sterling G and Cao, Yang and Darvish, Kourosh and Hao, Han and Lo, Stanley and Pablo-Garc{\'\i}a, Sergio and Rajaonson, Ella M and Skreta, Marta and others},
  journal={Chemical Reviews},
  volume={124},
  number={16},
  pages={9633--9732},
  year={2024},
  publisher={ACS Publications}
}

@article{zhu2022all,
  title={An all-round AI-Chemist with a scientific mind},
  author={Zhu, Qing and Zhang, Fei and Huang, Yan and Xiao, Hengyu and Zhao, LuYuan and Zhang, XuChun and Song, Tao and Tang, XinSheng and Li, Xiang and He, Guo and others},
  journal={National Science Review},
  volume={9},
  number={10},
  pages={nwac190},
  year={2022},
  publisher={Oxford University Press}
}

@inproceedings{yangmoose,
  title={MOOSE-Chem: Large Language Models for Rediscovering Unseen Chemistry Scientific Hypotheses},
  author={Yang, Zonglin and Liu, Wanhao and Gao, Ben and Xie, Tong and Li, Yuqiang and Ouyang, Wanli and Poria, Soujanya and Cambria, Erik and Zhou, Dongzhan},
  booktitle={ICLR},
  year={2025}
}

@article{huang2025deep,
  title={Deep research agents: A systematic examination and roadmap},
  author={Huang, Yuxuan and Chen, Yihang and Zhang, Haozheng and Li, Kang and Zhou, Huichi and Fang, Meng and Yang, Linyi and Li, Xiaoguang and Shang, Lifeng and Xu, Songcen and others},
  journal={arXiv preprint arXiv:2506.18096},
  year={2025}
}

@inproceedings{huangautomated,
  title={Automated Hypothesis Validation with Agentic Sequential Falsifications},
  author={Huang, Kexin and Jin, Ying and Li, Ryan and Li, Michael Y and Candes, Emmanuel and Leskovec, Jure},
  booktitle={ICML},
  year={2025}
}

@article{gottweis2025towards,
  title={Towards an AI co-scientist},
  author={Gottweis, Juraj and Weng, Wei-Hung and Daryin, Alexander and Tu, Tao and Palepu, Anil and Sirkovic, Petar and Myaskovsky, Artiom and Weissenberger, Felix and Rong, Keran and Tanno, Ryutaro and others},
  journal={arXiv preprint arXiv:2502.18864},
  year={2025}
}

@article{novikov2025alphaevolve,
  title={AlphaEvolve: A coding agent for scientific and algorithmic discovery},
  author={Novikov, Alexander and V{\~u}, Ng{\^a}n and Eisenberger, Marvin and Dupont, Emilien and Huang, Po-Sen and Wagner, Adam Zsolt and Shirobokov, Sergey and Kozlovskii, Borislav and Ruiz, Francisco JR and Mehrabian, Abbas and others},
  journal={arXiv preprint arXiv:2506.13131},
  year={2025}
}

@article{Lange2025ShinkaEvolveTO,
  title={ShinkaEvolve: Towards Open-Ended And Sample-Efficient Program Evolution},
  author={Robert Tjarko Lange and Yuki Imajuku and Edoardo Cetin},
  journal={ArXiv},
  year={2025},
  volume={abs/2509.19349},
}

@article{Wang2025ThetaEvolveTL,
  title={ThetaEvolve: Test-time Learning on Open Problems},
  author={Yiping Wang and Shao-Rong Su and Zhiyuan Zeng and Eva Xu and Liliang Ren and Xinyu Yang and Zeyi Huang and Xuehai He and Luyao Ma and Baolin Peng and Hao Cheng and Pengcheng He and Weizhu Chen and Shuohang Wang and Simon Shaolei Du and Yelong Shen},
  journal={ArXiv},
  year={2025},
  volume={abs/2511.23473},
}

@article{ZHENG2025FromAT,
  title={From Automation to Autonomy: A Survey on Large Language Models in Scientific Discovery},
  author={Tianshi ZHENG and Zheye Deng and Hong Ting Tsang and Weiqi Wang and Jiaxin Bai and Zihao Wang and Yangqiu Song},
  journal={ArXiv},
  year={2025},
  volume={abs/2505.13259},
}

@article{Mitchener2025KosmosAA,
  title={Kosmos: An AI Scientist for Autonomous Discovery},
  author={Ludovico Mitchener and Angela Yiu and Benjamin Chang and Mathieu Bourdenx and Tyler Nadolski and Arvis Sulovari and Eric C. Landsness and D{\'a}niel L. Barab{\'a}si and Siddharth Narayanan and Nicky Evans and Shriya Reddy and Martha S. Foiani and Aizad Kamal and Leah P. Shriver and Fang Cao and Asmamaw T. Wassie and Jon M. Laurent and Edwin Melville-Green and Mayk Caldas Ramos and Albert Bou and Kaleigh F. Roberts and Sladjana Zagorac and Timothy C. Orr and Miranda E. Orr and Kevin J. Zwezdaryk and Ali E. Ghareeb and Laurie McCoy and Bruna Gomes and Euan A Ashley and Karen E. Duff and Tonio Buonassisi and Tom Rainforth and Randall J. Bateman and Michael Skarlinski and Samuel G. Rodriques and Michaela M. Hinks and Andrew D. White},
  journal={ArXiv},
  year={2025},
  volume={abs/2511.02824},
}

@inproceedings{causalcoat2024,
     author = {Liu, Chenxi and Chen, Yongqiang and Liu, Tongliang and Gong, Mingming and Cheng, James and Han, Bo and Zhang, Kun},
     booktitle = {Advances in Neural Information Processing Systems},
     editor = {A. Globerson and L. Mackey and D. Belgrave and A. Fan and U. Paquet and J. Tomczak and C. Zhang},
     pages = {102307--102365},
     publisher = {Curran Associates, Inc.},
     title = {Discovery of the Hidden World with Large Language Models},
     url = {https://proceedings.neurips.cc/paper_files/paper/2024/file/b99a07486702417d3b1bd64ec2cf74ad-Paper-Conference.pdf},
     volume = {37},
     year = {2024}
}

@article{Li2025CanLL,
    title={Can Large Language Models Help Experimental Design for Causal Discovery?},
    author={Junyi Li and Yongqiang Chen and Chenxi Liu and Qianyi Cai and Tongliang Liu and Bo Han and Kun Zhang and Hui Xiong},
    year={2025},
    journal={ArXiv},
    volume={abs/2503.01139},
    url={https://arxiv.org/abs/2503.01139}
}

@article{chan2024mle-bench,
  title={MLE-bench: Evaluating Machine Learning Agents on Machine Learning Engineering},
  author={Jun Shern Chan and Neil Chowdhury and Oliver Jaffe and James Aung and Dane Sherburn and Evan Mays and Giulio Starace and Kevin Liu and Leon Maksin and Tejal Patwardhan and Lilian Weng and Aleksander Madry},
  year={2024},
  eprint={2410.07095},
  archivePrefix={arXiv},
  primaryClass={cs.CL},
  url={https://arxiv.org/abs/2410.07095}
}

@inproceedings{
verma2025causal,
title={Causal {AI} Scientist: Facilitating Causal Data Science with Large Language Models},
author={Vishal Verma and Sawal Acharya and Devansh Bhardwaj and Samuel Simko and Yongjin Yang and Anahita Haghighat and Dominik Janzing and Mrinmaya Sachan and Bernhard Sch{\"o}lkopf and Zhijing Jin},
booktitle={NeurIPS 2025 Workshop on CauScien: Uncovering Causality in Science},
year={2025},
url={https://openreview.net/forum?id=EDWTHMVOCj}
}


%% file: ref_1_llm.bib
@misc{xai_grok4_1_fast_2025,
  author       = {{xAI}},
  title        = {Grok 4.1 Fast and Agent Tools API},
  year         = {2025},
  date         = {2025-11-19},
  url          = {https://x.ai/news/grok-4-1-fast},
  organization = {xAI},
  note         = {Accessed: 2026-02-10},
}


%% file: ref_2_causality.bib
@article{perry2022causal,
  title={Causal discovery in heterogeneous environments under the sparse mechanism shift hypothesis},
  author={Perry, Ronan and Von K{\"u}gelgen, Julius and Sch{\"o}lkopf, Bernhard},
  journal={Advances in Neural Information Processing Systems},
  volume={35},
  pages={10904--10917},
  year={2022}
}

@article{mooij2020joint,
  title={Joint causal inference from multiple contexts},
  author={Mooij, Joris M and Magliacane, Sara and Claassen, Tom},
  journal={Journal of machine learning research},
  volume={21},
  number={99},
  pages={1--108},
  year={2020}
}

@article{brouillard2020differentiable,
  title={Differentiable causal discovery from interventional data},
  author={Brouillard, Philippe and Lachapelle, S{\'e}bastien and Lacoste, Alexandre and Lacoste-Julien, Simon and Drouin, Alexandre},
  journal={Advances in Neural Information Processing Systems},
  volume={33},
  pages={21865--21877},
  year={2020}
}

@inproceedings{yang2018characterizing,
  title={Characterizing and learning equivalence classes of causal dags under interventions},
  author={Yang, Karren and Katcoff, Abigail and Uhler, Caroline},
  booktitle={International Conference on Machine Learning},
  pages={5541--5550},
  year={2018},
  organization={PMLR}
}

@inproceedings{li2025efficient,
  title={Efficient and Trustworthy Causal Discovery with Latent Variables and Complex Relations},
  author={Li, Xiu-Chuan and Liu, Tongliang},
  booktitle={The Thirteenth International Conference on Learning Representations},
  year={2025}
}

@inproceedings{li2025recovery,
  title={Recovery of causal graph involving latent variables via homologous surrogates},
  author={Li, Xiu-Chuan and Wang, Jun and Liu, Tongliang},
  booktitle={The Thirteenth International Conference on Learning Representations},
  year={2025}
}

@inproceedings{liu2023causal,
  title={Causal structure learning for latent intervened non-stationary data},
  author={Liu, Chenxi and Kuang, Kun},
  booktitle={International Conference on Machine Learning},
  pages={21756--21777},
  year={2023},
  organization={PMLR}
}

@article{huang2020causal,
  title={Causal discovery from heterogeneous/nonstationary data},
  author={Huang, Biwei and Zhang, Kun and Zhang, Jiji and Ramsey, Joseph and Sanchez-Romero, Ruben and Glymour, Clark and Sch{\"o}lkopf, Bernhard},
  journal={Journal of Machine Learning Research},
  volume={21},
  number={89},
  pages={1--53},
  year={2020}
}

@inproceedings{huang2019causal,
  title={Causal discovery and forecasting in nonstationary environments with state-space models},
  author={Huang, Biwei and Zhang, Kun and Gong, Mingming and Glymour, Clark},
  booktitle={International conference on machine learning},
  pages={2901--2910},
  year={2019},
  organization={Pmlr}
}

@inproceedings{malinsky2019learning,
  title={Learning the structure of a nonstationary vector autoregression},
  author={Malinsky, Daniel and Spirtes, Peter},
  booktitle={The 22nd International Conference on Artificial Intelligence and Statistics},
  pages={2986--2994},
  year={2019},
  organization={PMLR}
}

@article{hoyer2008nonlinear,
  title={Nonlinear causal discovery with additive noise models},
  author={Hoyer, Patrik and Janzing, Dominik and Mooij, Joris M and Peters, Jonas and Sch{\"o}lkopf, Bernhard},
  journal={Advances in neural information processing systems},
  volume={21},
  year={2008}
}

@article{zhang2012identifiability,
  title={On the identifiability of the post-nonlinear causal model},
  author={Zhang, Kun and Hyvarinen, Aapo},
  journal={arXiv preprint arXiv:1205.2599},
  year={2012}
}

@article{shimizu2006linear,
  title={A linear non-Gaussian acyclic model for causal discovery.},
  author={Shimizu, Shohei and Hoyer, Patrik O and Hyv{\"a}rinen, Aapo and Kerminen, Antti and Jordan, Michael},
  journal={Journal of Machine Learning Research},
  volume={7},
  number={10},
  year={2006}
}

@inproceedings{spirtes1995causal,
  title={Causal inference in the presence of latent variables and selection bias},
  author={Spirtes, Peter and Meek, Christopher and Richardson, Thomas},
  booktitle={Proceedings of the Eleventh conference on Uncertainty in artificial intelligence},
  pages={499--506},
  year={1995}
}

@book{pearl2009causality,
  title={Causality},
  author={Pearl, Judea},
  year={2009},
  publisher={Cambridge university press}
}

@article{greenland1999causal,
  title={Causal diagrams for epidemiologic research},
  author={Greenland, Sander and Pearl, Judea and Robins, James M},
  journal={Epidemiology},
  volume={10},
  number={1},
  pages={37--48},
  year={1999},
  publisher={LWW}
}

@book{spirtes2000causation,
  title={Causation, prediction, and search},
  author={Spirtes, Peter and Glymour, Clark N and Scheines, Richard},
  year={2000},
  publisher={MIT press}
}

@inproceedings{li2024realtcd,
  title={Realtcd: Temporal causal discovery from interventional data with large language model},
  author={Li, Peiwen and Wang, Xin and Zhang, Zeyang and Meng, Yuan and Shen, Fang and Li, Yue and Wang, Jialong and Li, Yang and Zhu, Wenwu},
  booktitle={Proceedings of the 33rd ACM International Conference on Information and Knowledge Management},
  pages={4669--4677},
  year={2024}
}

@article{shen2024exploring,
  title={Exploring multi-modal integration with tool-augmented llm agents for precise causal discovery},
  author={Shen, C and Chen, Zhengzhang and Luo, Dongsheng and Xu, Dongkuan and Chen, Haifeng and Ni, Jingchao},
  journal={arXiv preprint arXiv:2412.13667},
  volume={1},
  number={3},
  year={2024}
}

@article{khatibi2024alcm,
  title={Alcm: Autonomous llm-augmented causal discovery framework},
  author={Khatibi, Elahe and Abbasian, Mahyar and Yang, Zhongqi and Azimi, Iman and Rahmani, Amir M},
  journal={arXiv preprint arXiv:2405.01744},
  year={2024}
}

@article{long2023causal,
  title={Causal discovery with language models as imperfect experts},
  author={Long, Stephanie and Pich{\'e}, Alexandre and Zantedeschi, Valentina and Schuster, Tibor and Drouin, Alexandre},
  journal={arXiv preprint arXiv:2307.02390},
  year={2023}
}

@article{ban2023causal,
  title={Causal structure learning supervised by large language model},
  author={Ban, Taiyu and Chen, Lyuzhou and Lyu, Derui and Wang, Xiangyu and Chen, Huanhuan},
  journal={arXiv preprint arXiv:2311.11689},
  year={2023}
}

@inproceedings{lirevealing,
  title={Revealing Multimodal Causality with Large Language Models},
  author={Li, Jin and Wang, Shoujin and Zhang, Qi and Liu, Feng and Liu, Tongliang and Cao, Longbing and Yu, Shui and Chen, Fang},
  booktitle={The Thirty-ninth Annual Conference on Neural Information Processing Systems},
  year={2025}
}

@inproceedings{abdulaal2023causal,
  title={Causal modelling agents: Causal graph discovery through synergising metadata-and data-driven reasoning},
  author={Abdulaal, Ahmed and Montana-Brown, Nina and He, Tiantian and Ijishakin, Ayodeji and Drobnjak, Ivana and Castro, Daniel C and Alexander, Daniel C and others},
  booktitle={The Twelfth International Conference on Learning Representations},
  year={2023}
}

@article{wang2025causal,
  title={Causal-copilot: An autonomous causal analysis agent},
  author={Wang, Xinyue and Zhou, Kun and Wu, Wenyi and Singh, Har Simrat and Nan, Fang and Jin, Songyao and Philip, Aryan and Patnaik, Saloni and Zhu, Hou and Singh, Shivam and others},
  journal={arXiv preprint arXiv:2504.13263},
  year={2025}
}

@inproceedings{sheth2025can,
  title={Can LLMs Propose Instrumental Variables for Causal Reasoning?},
  author={Sheth, Ivaxi and Jin, Zhijing and Wilder, Bryan and Janzing, Dominik and Fritz, Mario},
  booktitle={NeurIPS 2025 Workshop on CauScien: Uncovering Causality in Science}
}

@misc{coat2025discoveringreasoningcausalityhidden,
      title={Discovering and Reasoning of Causality in the Hidden World with Large Language Models}, 
      author={Chenxi Liu and Yongqiang Chen and Tongliang Liu and Mingming Gong and James Cheng and Bo Han and Kun Zhang},
      year={2025},
      eprint={2402.03941},
      archivePrefix={arXiv},
      primaryClass={cs.LG},
      url={https://arxiv.org/abs/2402.03941}, 
}

@article{dhawan2024end,
  title={End-to-end causal effect estimation from unstructured natural language data},
  author={Dhawan, Nikita and Cotta, Leonardo and Ullrich, Karen and Krishnan, Rahul G and Maddison, Chris J},
  journal={Advances in Neural Information Processing Systems},
  volume={37},
  pages={77165--77199},
  year={2024}
}

@article{vashishtha2023causal,
  title={Causal ordering as a robust interface for integrating expert knowledge},
  author={Vashishtha, Siddharth and others},
  journal={Advances in Neural Information Processing Systems},
  year={2023}
}

@article{jiralerspong2024efficient,
  title={Efficient causal graph discovery using large language models},
  author={Jiralerspong, Krittanut and others},
  journal={arXiv preprint arXiv:2402.01207},
  year={2024}
}

@misc{ClarkOutline,
  author       = {Glymour, Clark},
  title        = {An Outline of the History of Methods of Discovering Causality},
  institution  = {Carnegie Mellon University, Department of Philosophy},
  date = {n.d.},
  url          = {https://www.cmu.edu/dietrich/philosophy/docs/glymour/an-outline-of-the-history-of-methods-of-discovering-causality.pdf},
  note         = {Accessed: 2026-01-29}
}

@book{wallace1981causality,
  title={Causality and Scientific Explanation},
  author={Wallace, W.A.},
  number={v. 2},
  isbn={9780819114815},
  lccn={81004834},
  series={Causality and Scientific Explanation},
  year={1981},
  publisher={University Press of America}
}

@article{sci_revolution,
  title={The Structure of Scientific Revolutions},
  author={Thomas S. Kuhn and David Hawkins},
  journal={American Journal of Physics},
  year={1963},
  volume={31},
  pages={554-555},
}

@article{khemakhem2020ice,
  title={Ice-beem: Identifiable conditional energy-based deep models based on nonlinear ica},
  author={Khemakhem, Ilyes and Monti, Ricardo and Kingma, Diederik and Hyvarinen, Aapo},
  journal={Conference and Workshop on Neural Information Processing Systems},
  volume={33},
  pages={12768--12778},
  year={2020}
}

@book{spirtes2000cps,
  title={Causation, prediction, and search},
  author={Spirtes, Peter and Glymour, Clark N and Scheines, Richard},
  year={2000},
  publisher={MIT press}
}

@InCollection{sep-abduction,
	author       =	{Douven, Igor},
	title        =	{{Abduction}},
	booktitle    =	{The {Stanford} Encyclopedia of Philosophy},
	editor       =	{Edward N. Zalta and Uri Nodelman},
	howpublished =	{\url{https://plato.stanford.edu/archives/win2025/entries/abduction/}},
	year         =	{2025},
	edition      =	{{W}inter 2025},
	publisher    =	{Metaphysics Research Lab, Stanford University}
}


%% file: ref_3_others.bib
@article{kaelbling1998planning,
  title={Planning and acting in partially observable stochastic domains},
  author={Kaelbling, Leslie Pack and Littman, Michael L and Cassandra, Anthony R},
  journal={Artificial intelligence},
  volume={101},
  number={1-2},
  pages={99--134},
  year={1998},
  publisher={Elsevier}
}

@article{guo2025deepseek,
  title={Deepseek-r1: Incentivizing reasoning capability in llms via reinforcement learning},
  author={Guo, Daya and Yang, Dejian and Zhang, Haowei and Song, Junxiao and Zhang, Ruoyu and Xu, Runxin and Zhu, Qihao and Ma, Shirong and Wang, Peiyi and Bi, Xiao and others},
  journal={arXiv preprint arXiv:2501.12948},
  year={2025}
}

@article{li2025system,
  title={From system 1 to system 2: A survey of reasoning large language models},
  author={Li, Zhong-Zhi and Zhang, Duzhen and Zhang, Ming-Liang and Zhang, Jiaxin and Liu, Zengyan and Yao, Yuxuan and Xu, Haotian and Zheng, Junhao and Wang, Pei-Jie and Chen, Xiuyi and others},
  journal={arXiv preprint arXiv:2502.17419},
  year={2025}
}

@article{plaat2025agentic,
  title={Agentic large language models, a survey},
  author={Plaat, Aske and van Duijn, Max and van Stein, Niki and Preuss, Mike and van der Putten, Peter and Batenburg, Kees Joost},
  journal={arXiv preprint arXiv:2503.23037},
  year={2025}
}

@book{datasetshift,
  title     = {Dataset shift in machine learning},
  author    = {Quinonero-Candela, Joaquin and Sugiyama, Masashi and Schwaighofer, Anton and Lawrence, Neil D},
  year      = {2008},
  publisher = {Mit Press}
}

@inproceedings{feat,
    title={Understanding and Improving Feature Learning for Out-of-Distribution Generalization},
    author={Yongqiang Chen and Wei Huang and Kaiwen Zhou and Yatao Bian and Bo Han and James Cheng},
    booktitle={Advances in Neural Information Processing Systems},
    year={2023},
}
